\documentclass[runningheads]{llncs}
\usepackage[T1]{fontenc}
\usepackage{graphicx}
\usepackage{booktabs}
\usepackage[misc]{ifsym}
\newcommand{\corr}{(\Letter)}
\usepackage{mwe}

\usepackage{url}
\usepackage{multirow}
\usepackage{algorithm}%
\usepackage{algorithmicx}%
\usepackage{algpseudocode}%
\usepackage{listings}%

\usepackage{hyperref}
\usepackage[caption=false]{subfig}
\usepackage{pdfpages}
\usepackage{rotating}

\usepackage{amsmath}
\usepackage{amsfonts}

\usepackage{makecell}
\usepackage{enumitem}

\usepackage{subcaption}
\usepackage{float} 

\begin{document}




\title{Understanding Latent Flow Models for Tabular Data Synthesis: Targets, Paths, and Sampling}

\titlerunning{Latent Flow Models for Tabular Synthesis}



\author{Bahrul Ilmi Nasution \corr
} 

\institute{Department of Social Statistics, The University of Manchester, UK \\ \email{bahrul.nasution@manchester.ac.uk}
}



\maketitle              
\begingroup

\renewcommand{\thefootnote}{}
\footnotetext{Accepted at European Conference on Machine Learning and Principles and Practice of Knowledge Discovery in Databases (ECML-PKDD) 2026, Applied Data Science Track. This is the author's accepted manuscript.}
\endgroup

\begin{abstract}

Synthetic tabular data enables microdata sharing in regulated domains, yet deploying continuous-time generative models requires balancing analytical utility, disclosure risk, and computational cost. Latent-space flow models are flexible, but theoretical equivalences across learning targets, probability paths, and sampling dynamics can translate into different behaviour under finite-step integration and explicit compute budgets. We present an empirical study of tabular latent flow models across seven datasets, evaluating velocity, score, noise, and posterior matching objectives under optimal transport (OT) and variance-preserving (VP) paths, ODE and SDE sampling, and varying integration budgets. Our contributions are threefold: (1) we show that the learning target largely determines the utility–risk operating regime, with velocity and posterior matching tending to yield higher utility, while score and noise matching tend to achieve lower disclosure risk; (2) we demonstrate that configuration and sampling choices shift performance, with midpoint often improving distributional fidelity and OT paths often tolerating earlier stopping than VP, enabling compute savings under fixed budgets or risk thresholds; and (3) we distil these findings into actionable defaults and practical configuration guidance to support pre-release model selection under disclosure risk and resource constraints. The code implementation and supplementary materials can be accessed in \url{https://github.com/rulnasution/tabular-latent-flow/}.



\end{abstract}

\section{Introduction}
\label{sec:intro}

Data owners, such as National Statistics Offices (NSOs) and financial institutions, face a persistent dilemma: enabling external researchers and policymakers to analyse microdata without revealing individual information. Tabular data synthesis addresses this challenge by generating data that preserves statistical properties but has different records from the original data~\cite{Little2024synthetic}. However, compared with images and text, tabular data poses distinct generative challenges. The tabular generative model mixes continuous and categorical variables, exhibits complex dependencies, and requires understanding of the relationship between utility and disclosure risk~\cite{zhang2024mixedtype,Little2022comparing}. Mastering these challenges is central for data synthesis in regulated domains.

Deep generative models have recently improved the quality of synthetic tabular data, with flow-based approaches, such as flow matching (FM), becoming increasingly prominent~\cite{guzman-cordero2025exponential,nasution2025flowmatchingtabulardata}. FM~\cite{albergo2023stochasticinterpolantsunifyingframework,lipman2023flowmatchinggenerativemodeling} is a continuous-time generative framework that learns a vector field transporting a simple base distribution to the data distribution through an ordinary differential equation, while supporting simulation-free training. Flow-based and diffusion-based models can also be trained in latent spaces, where the generative dynamics operate on continuous representations that are later decoded back to the data space using a variational autoencoder (VAE)~\cite{dao2023flowmatchinglatentspace,rombach2022highresolutionimagesynthesislatent}. For tabular data, recent latent-space generators adopt this strategy by training diffusion~\cite{zhang2024mixedtype} or flow-matching models~\cite{nasution2025flowmatchingtabulardata} on tabular latent representations. By mapping mixed-type tabular records into a continuous latent space, these approaches can simplify generative modelling and reduce the cost of operating directly on heterogeneous features, which has competitive performance in prior works~\cite{zhang2024mixedtype,nasution2025flowmatchingtabulardata}.

Existing theory on diffusion and flow-based generative models shows that, under Gaussian probability paths, velocity, score, and noise matching are analytically connected~\cite{lipman2023flowmatchinggenerativemodeling,holderrieth2025lecturenotes}. Similarly, reverse-SDE samplers and probability-flow ODE samplers can share the same continuous-time marginal distributions~\cite{song2021maximum}. Meanwhile, posterior matching, or variational flow matching (VFM), is also closely related to flow matching in continuous Gaussian settings, while introducing a variational training formulation and an associated stochastic sampling procedure~\cite{eijkelboom2024variationalflowmatchinggraph,guzman-cordero2025exponential}. However, these theoretical connections do not determine which configuration works best in finite-step latent tabular synthesis, where models are learned approximately, sampled under limited compute, and trained using different objectives. Thus, an empirical evaluation of learning targets, probability paths, sampling dynamics, and numerical solvers becomes essential in practice.

This work addresses these questions through a systematic empirical study of latent tabular flow models. In this study, latent tabular flow models refer to latent continuous-time generative models defined by a probability path and sampled through ODE or SDE dynamics. Flow, noise, score, and posterior matching denote different learning-target parameterisations within this shared setup, rather than separate model families. We evaluate representative target--path configurations across datasets and quantify utility, disclosure risk, and efficiency under controlled changes in learning target, probability path, sampling dynamics, solver family, integration endpoint, and integration steps. The goal is to provide practitioner-facing guidance on how to configure latent tabular flow models for different circumstances. Specifically, we study the following research questions:

\begin{itemize}
    \item \textbf{Q1}: Does training with different learning targets (velocity, score, noise, and posterior matching) produce empirically equivalent results in tabular latent flow models, as suggested by theory? 

    \item \textbf{Q2}: How does the choice of probability path (OT versus VP), affect utility and disclosure risk across learning targets? 

    \item \textbf{Q3}: How do sampling dynamics and numerical solvers compare across trained tabular latent flow models in terms of data quality and disclosure risk? 

    \item \textbf{Q4}: How does the integration endpoint affect the quality of the synthetic data for each target-path configuration? 

    \item \textbf{Q5}: How do integration steps affect the synthetic data quality for each target-path configuration? 
\end{itemize}

Together, these questions clarify how theoretical modeling choices translate into empirical behavior in tabular data synthesis, and they support informed configuration decisions under competing utility, disclosure risk, and computational constraints.

\section{Related Work}
\label{sec:related}

\textbf{Flow and diffusion models.}
Diffusion models learn reverse-time dynamics through SDEs and are strong baselines for high-quality generation~\cite{ho2020denoising,song2019generative,song2021maximum}. Instead, FM learns a deterministic velocity field, with conditional FM enabling simulation-free training~\cite{lipman2023flowmatchinggenerativemodeling,liu2022flowstraightfastlearning}. Both frameworks are based on probability path design, where OT (for FM) and VP (for diffusion) paths are common choices and ODE/SDE sampling can be closely related under Gaussian paths~\cite{albergo2023stochasticinterpolantsunifyingframework,lipman2023flowmatchinggenerativemodeling,holderrieth2025lecturenotes}. VFM further connects velocity learning to variational inference and score-based objectives~\cite{eijkelboom2024variationalflowmatchinggraph,guzman-cordero2025exponential}. Recent work also models latent distributions learnt by VAEs using diffusion or flow models, including in tabular settings~\cite{rombach2022highresolutionimagesynthesislatent,dao2023flowmatchinglatentspace,zhang2024mixedtype,nasution2025flowmatchingtabulardata}.

\noindent\textbf{Tabular synthesis and practical design choices.}
Tabular synthesis should handle mixed feature types, limited data, and also utility and risk. Diffusion-based approaches such as TabDDPM and latent diffusion variants (e.g., TabSyn) demonstrate strong utility and fidelity~\cite{kotelnikov2023tabddpm,zhang2024mixedtype,mueller2025continuous,villaiz2025diffusion}. On the other hand, FM and VFM provide a more sampling-efficient alternative in both latent and data space on tabular benchmarks~\cite{nasution2025flowmatchingtabulardata,guzman-cordero2025exponential}. However, within a unified latent continuous-time setup, the practical impact of choosing learning targets, paths (OT vs.\ VP), deterministic versus stochastic sampling, and solver families remains less systematically studied. Our work addresses this by benchmarking these choices under controlled sampling budgets for tabular latent flow models.

\section{Background}
\label{sec:background}

This section discusses the essentials for our study, covering flow model fundamentals and probability paths (Section~\ref{subsec:fm_cfm}), followed by a brief understanding of learning targets including velocity, score, noise, and posterior (Section~\ref{subsec:targets_vfm_short}).

\subsection{Flow Model and Conditional Gaussian Paths}
\label{subsec:fm_cfm}

A flow model generates samples by \emph{moving} a point over time from an initial distribution (usually a standard Gaussian) to the data distribution~\cite{lipman2023flowmatchinggenerativemodeling,albergo2023stochasticinterpolantsunifyingframework}. In common practice, we start from noise $x_0 \sim \mathcal{N}(0,I)$ and evolve it with a learnt time-dependent vector field $v_t^\theta$ to obtain data $x_1 \sim p_{data}$ based on ODE:
\begin{equation}
\frac{d x_t}{dt} = v_t^\theta(x_t), \quad t \in [0,1].
\label{eq:ode_gen}
\end{equation}

Intuitively, $v_t^\theta$ tells each point ``where to move next'' so that the evolving distribution matches the data distribution at $t=1$. Training $v_t^\theta$ can use \emph{conditional} flow~\cite{holderrieth2025lecturenotes}, which constructs a conditional probability path anchored at an endpoint $x_1\sim p_{\text{data}}$~\cite{lipman2023flowmatchinggenerativemodeling}. The most convenient choice is a Gaussian conditional path:
\begin{equation}
x_t=\alpha_t x_1+\sigma_t x_0,\quad x_0\sim \mathcal{N}(0,I),
\label{eq:gaussian_path}
\end{equation}

which implies $p_t(x_t\mid x_1)=\mathcal{N}(\alpha_t x_1,\sigma_t^2 I)$. The path is defined by interpolants $(\alpha_t,\sigma_t)$. There are two common interpolants, namely optimal transport (OT) which is defined as a straight path and variance-preserving diffusion (VP) path~\cite{lipman2023flowmatchinggenerativemodeling}. Both paths are defined as
\begin{align}
    \alpha_t^{\text{OT}}=t,\quad \sigma_t^{\text{OT}}=1-(1-\sigma_{\min})t \\
\alpha_t^{\text{VP}}=\sqrt{\bar{\alpha}(1-t)}, \quad
\sigma_t^{\text{VP}}=\sqrt{1-(\alpha_t^{\text{VP}})^2}
\end{align}

with $\bar{\alpha}(t)=\exp\!\Big(-\int_0^{t}\beta(s)\,ds\Big)$. The interpolants move the process from a noise-dominated state at $t=0$ toward the data endpoint at $t=1$. Generation reduces to solving  ODE in~\eqref{eq:ode_gen} (or an SDE variant), so the main engineering questions are what the network should learn, which path to follow, and how the solver and compute budget affect quality.




\subsection{Targets: Velocity, Score, Noise, and Posterior}
\label{subsec:targets_vfm_short}

Differentiating the Gaussian probability path in Equation~\eqref{eq:gaussian_path} with respect to time yields a tractable conditional velocity target~\cite{lipman2023flowmatchinggenerativemodeling}:
\begin{equation}
u_t(x_t\mid x_1)=\dot{\alpha}_t x_1+\dot{\sigma}_t x_0.
\label{eq:cond_velocity}
\end{equation}
Flow matching regresses the model velocity $v_t^\theta(x_t)$ to $u_t(x_t\mid x_1)$. Furthermore, the same Gaussian probability path also provides a closed-form conditional score,
\begin{equation}
\nabla_{x_t}\log p_t(x_t\mid x_1)=-\frac{x_t-\alpha_t x_1}{\sigma_t^2},
\label{eq:cond_score}
\end{equation}
which leads to score matching by regressing a predicted score $s_t^\theta(x_t)$ to this target~\cite{song2019generative}. Finally, DDPM-style noise prediction follows directly from Equation~\eqref{eq:gaussian_path} by regressing the predicted noise $\epsilon_t^\theta(x_t)$ with the initial noise, $x_0=(x_t-\alpha_t x_1)/\sigma_t$~\cite{ho2020denoising,lipman2023flowmatchinggenerativemodeling}. 
These parameterisations are related under Gaussian paths, differing mainly by time-dependent scalings. In practice, they can behave differently under finite-step sampling and limited compute budgets, which motivates our empirical comparisons.

While FM, SM, and NM specify a direct regression target at time $t$, variational flow matching (VFM)~\cite{eijkelboom2024variationalflowmatchinggraph,guzman-cordero2025exponential} replaces direct target regression with posterior inference over the endpoint. Specifically, VFM learns a variational posterior $q_t^\theta(x_1\mid x_t)$ and recovers the velocity by conditional expectation:
\begin{equation}
v_t^\theta(x_t)=\mathbb{E}_{q_t^\theta(x_1\mid x_t)}\!\left[u_t(x_t\mid x_1)\right].
\label{eq:vfm_velocity}
\end{equation}
In the Gaussian distribution, $q_t^\theta(x_1\mid x_t)$ is parameterised by a mean $\mu_t^\theta(x_t)$, which can be interpreted as predicting an endpoint. We show interface for deterministic (ODE) and stochastic (SDE) samplers in later sections.

\section{Designing Flow Models on Latent Variables: Theory and Implementation}
\label{sec:design_latent}

We study latent-space flow models for tabular synthesis by fixing the latent representation and varying a small set of design axes that affect utility, disclosure risk, and runtime. Given fixed model architecture, we vary five design axes that matter for configuration and evaluation:
\begin{itemize}
    \item \textbf{Learning target:} what the network predicts during training (velocity, score, noise, or posterior mean).
    \item \textbf{Path:} OT versus VP, which determines the interpolants $(\alpha_t,\sigma_t)$.
    \item \textbf{Sampling dynamics:} deterministic ODE versus stochastic SDE.
    \item \textbf{Solver family:} Euler vs Midpoint.
    \item \textbf{Sampling budget:} integration endpoint and steps.
\end{itemize}





\subsection{Training Pipeline}
\label{subsec:latent_training}


Let \(z_1 \sim p_{\mathcal{Z}}\) denote the latent representation of a preprocessed real record \(x^{\text{data}}\), obtained from a pretrained VAE encoder, \(z_1=\mathrm{Enc}(x^{\text{data}})\). During training, we sample \(z_0 \sim \mathcal{N}(0,I)\) and \(t \sim \mathrm{Uniform}(0,1)\), then construct the interpolated latent state
\begin{equation}
z_t = \alpha_t z_1 + \sigma_t z_0 .
\label{eq:latent_path}
\end{equation}
Let \(\text{F}_t^\theta(z_t)=\mathrm{NN}_\theta(z_t,t)\) denote the network output. All target parameterisations use this same latent interpolation, but differ in the training target and in how the velocity and score components required for sampling are recovered, as summarised in Table~\ref{tab:latent_objectives}. The full training algorithm is provided in Appendix~\ref{sec:algorithm}.


\subsection{Learning Targets}
\label{subsec:objectives_table}

First, note two tractable conditional quantities under Equation~\eqref{eq:latent_path}:
\begin{equation}
u_t(z_t\mid z_1)=\dot{\alpha}_t z_1+\dot{\sigma}_t z_0,
\label{eq:cond_velocity_z}
\end{equation}
\begin{equation}
\nabla_{z_t}\log p_t(z_t\mid z_1)=-\frac{z_t-\alpha_t z_1}{\sigma_t^2}.
\label{eq:cond_score_z}
\end{equation}
We use these as targets for training FM and SM, respectively. In contrast, NM uses $z_0$ and VFM uses endpoint $z_1$ as a target. At sampling time, ODE solvers require a velocity field, while SDE solvers require both a velocity and a score. The mapping used for each variant is defined in Section~\ref{subsec:sampling_maps} and listed in Table~\ref{tab:latent_objectives}.

\begin{table}[ht]
\centering
\small
\setlength{\tabcolsep}{4pt}
\renewcommand{\arraystretch}{1.15}
\caption{Learning target for latent flow models. The table shows what the model predicts, the training loss, and how we obtain the components required by ODE and SDE sampling.}
\label{tab:latent_objectives}
\begin{tabular}{p{1cm} p{2.5cm} w{c}{3.8cm} p{1.5cm} p{1.5cm}}
\toprule
{Flow model} & {Model output $\text{F}_t^\theta$} & {Training loss} & {$v_t^\theta$} & {$s_t^\theta$} \\
\midrule
FM &
Velocity $v_t^\theta(z_t)$ &
$\mathbb{E}\|\text{F}_t^\theta(z_t)-u_t(z_t\mid z_1)\|_2^2$ &
$\text{F}_t^\theta(z_t)$ &
Equation \eqref{eq:score_from_velocity} \\
\addlinespace
SM &
Scaled score $\text{F}_t^\theta(z_t)=\sigma_t\, s_t^\theta(z_t)$ &
$\mathbb{E}\big[\sigma_t^2\| \tfrac{\text{F}_t^\theta(z_t)}{\sigma_t}+\tfrac{z_t-\alpha_t z_1}{\sigma_t^2}\|_2^2\big]$ &
Equation \eqref{eq:velocity_from_score} &
$\text{F}_t^\theta/\sigma_t$ \\
\addlinespace
NM &
Noise $\epsilon_t^\theta(z_t)$ &
$\mathbb{E}\big[\|\text{F}_t^\theta(z_t)-z_0\|_2^2\big]$ &
Equation \eqref{eq:velocity_from_score} &
$-\text{F}_t^\theta/\sigma_t$ \\
\addlinespace
VFM &
Posterior mean $\mu_t^\theta(z_t)$&
$\mathbb{E}\big[\lambda(t)\|\text{F}_t^\theta(z_t)-z_1\|_2^2\big]$ &
Equation \eqref{eq:velocity_from_mean} &
Equation \eqref{eq:score_from_mean} \\
\bottomrule
\end{tabular}
\end{table}

\subsection{Sampling and Parameterisation Mappings}
\label{subsec:sampling_maps}

We generate a latent sample by drawing $\tilde z_0\sim\mathcal{N}(0,I)$ and simulating the ODE/SDE process to obtain latent samples $\tilde z_1$. Finally, latent samples generated by the flow model are passed through the fixed VAE decoder to reconstruct synthetic tabular records, followed by inverse preprocessing to obtain the original format of the records.

\subsubsection{A shared identity under Gaussian paths.}
\label{subsec:vel_score_prop}

The following identity links velocity and score under Gaussian paths and provides a principled conversion between parameterisations.


\begin{proposition}[Velocity--score equivalence; adapted from~\cite{holderrieth2025lecturenotes}]
\label{prop:vel_score}
Let $z_t = \alpha_t z_1 + \sigma_t z_0, \;\;z_0 \sim \mathcal{N}(0, I)$
define a Gaussian probability path. Then, for any $t$ such that $\alpha_t \neq 0$ and $\sigma_t > 0$, the conditional velocity field $u_t(z_t \mid z_1)$ and the conditional score function $\nabla_{z_t} \log p_t(z_t \mid z_1)$ satisfy
\begin{equation}
u_t(z_t \mid z_1)
=
\left(
\sigma_t^2 \frac{\dot{\alpha}_t}{\alpha_t}
-
\dot{\sigma}_t \sigma_t
\right)
\nabla_{z_t} \log p_t(z_t \mid z_1)
+
\frac{\dot{\alpha}_t}{\alpha_t} z_t.
\label{eq:vel_score_relation}
\end{equation}
\end{proposition}
The proof is in Appendix~\ref{ssec:proof-velscore} of the supplementary material. However, this theoretical equivalence does not necessarily imply identical finite-step sampling behaviour. In practice, model approximation error, discretisation error, solver choice, time-step schedules, and limited compute budgets can lead to different empirical behaviour under finite-step sampling. Proposition~\ref{prop:vel_score} allows us to compare targets fairly under matched sampling definitions.

\subsubsection{Deterministic sampling (ODE).}
\label{subsec:ode_sampling}

ODE sampling integrates
\begin{equation}
d\tilde z_t = v_t^\theta(\tilde z_t)\,dt,\quad \tilde z_0\sim\mathcal{N}(0,I).
\label{eq:ode_sampling}
\end{equation}
For FM, the velocity is predicted directly:
\begin{equation}
v_t^\theta(\tilde z_t)=\text{F}_t^\theta(\tilde z_t).
\label{eq:ode_velocity_fm}
\end{equation}
For SM or NM, we first obtain a score $s_t^\theta(\tilde z_t)$ as in Table~\ref{tab:latent_objectives}, then convert the score to velocity via Proposition~\ref{prop:vel_score}:
\begin{equation}
v_t^\theta(\tilde z_t)=
\Big(\sigma_t^2 \frac{\dot{\alpha}_t}{\alpha_t}-\dot{\sigma}_t \sigma_t\Big)s_t^\theta(\tilde z_t)
+\frac{\dot{\alpha}_t}{\alpha_t}\tilde z_t .
\label{eq:velocity_from_score}
\end{equation}
For VFM, the mean $\mu_t^\theta(\tilde z_t)$ yields the velocity of Equation~\eqref{eq:vfm_velocity}:
\begin{equation}
v_t^\theta(\tilde z_t)=
\Big(\dot{\alpha}_t-\frac{\dot{\sigma}_t}{\sigma_t}\alpha_t\Big)\mu_t^\theta(\tilde z_t)
+\frac{\dot{\sigma}_t}{\sigma_t}\tilde z_t .
\label{eq:velocity_from_mean}
\end{equation}

\subsubsection{Stochastic sampling (SDE).}
\label{subsec:sde_sampling}

SDE sampling integrates
\begin{equation}
d\tilde z_t =
\Big[v_t^\theta(\tilde z_t)+\tfrac{g_t^2}{2}s_t^\theta(\tilde z_t)\Big]dt + g_t\,dW_t,
\label{eq:sde_sampling}
\end{equation}

\noindent For SM and NM, $s_t^\theta$ is obtained directly from the model output as listed in Table~\ref{tab:latent_objectives}. In contrast, FM recovers the score by inverting Equation~\eqref{eq:vel_score_relation}:
\begin{equation}
s_t^\theta(\tilde z_t)=
\frac{\alpha_t v_t^\theta(\tilde z_t)-\dot{\alpha}_t \tilde z_t}
{\sigma_t^2 \dot{\alpha}_t-\alpha_t \dot{\sigma}_t \sigma_t}.
\label{eq:score_from_velocity}
\end{equation}
For VFM, the score can be calculated using Equation~\eqref{eq:score_from_mean}.
\begin{equation}
s_t^\theta(\tilde z_t)= -\frac{\tilde z_t-\alpha_t \mu_t^\theta(\tilde z_t)}{\sigma_t^2},
\label{eq:score_from_mean}
\end{equation}

Algorithm~\ref{alg:TabSynFlow-s} summarises the sampling process of latent flow models in general. In brief, sampling starts from Gaussian noise, integrates the learned latent dynamics using the selected ODE or SDE solver, and decodes the generated latent sample using the fixed VAE decoder.

    \begin{algorithm}[ht]
        \caption{Latent flow sample generation}
        \label{alg:TabSynFlow-s}
        \textbf{Input:} Trained latent flow model \(NN(\cdot)\), discrete step interval $\Delta$, trained decoder from VAE $\mathrm{Dec}(.)$, integrator (ODE or SDE), starting time $t_{\mathrm{start}}$, end time $t_{int}$ 
        \begin{algorithmic}[1]
            \State $\tilde{z}_0 \sim \mathcal{N}(0, I)$.
            \State $t=t_{\mathrm{start}}$
            \While{$t \le t_{int}$}
            \State $\text{F}_t^{\theta}(\tilde{z}_t) = \mathrm{NN}(\tilde{z}_t, t)$
            \State Calculate $v_t^\theta(\tilde{z}_t)$ \Comment{both ODE and SDE} 
            \If {SDE} Calculate {$s_t^\theta(\tilde{z}_t)$} \Comment{SDE only} \EndIf
            \State Update $\tilde{z}_{t+\Delta}$ using integrator.
            \State $t \textrm{+=} \Delta$
            \EndWhile
            \State $\tilde{x}^{\text{data}}=\mathrm{Dec}(\tilde{z}_{t_{int}})$ \Comment{Decode using pretrained VAE decoder}
        \end{algorithmic}
        \Return{$\tilde{\mathcal{D}}=\mathrm{InverseTransform}(\tilde{x}^{\text{data}})$}
    \end{algorithm}

\section{Experimental Results}
\label{sec:results}

This section reports empirical findings, covering the experiment setup and evaluation (Sections~\ref{sec:exp_setup} and ~\ref{sec:evaluation}), a comparison of learning targets, paths, sampling dynamics, and solvers (Section~\ref{sec:res-q123}), as well as integration time and step budget (Sections~\ref{sec:res-t-ode} and \ref{sec:res-nfe}).

\subsection{Experimental Setup}
\label{sec:exp_setup}

We adopt the experimental framework from previous works~\cite{nasution2025flowmatchingtabulardata,guzman-cordero2025exponential} with adaptations specific to the objectives of this study. Tables~\ref{tab:hyperparams_train} and~\ref{tab:baseline_hyperparams} in Appendix~\ref{sec:exp_details} of the supplementary material summarise the implementation details.

\textbf{Datasets. }We evaluate seven tabular datasets that span census and machine learning benchmarks (Table~\ref{tab:dataset}). We use five public census datasets (UK, Canada, Fiji, Rwanda, Indonesia), which can be obtained from the IPUMS repository~\cite{ipums}. On the other hand, we also provide additional comparisons using UCI datasets commonly used in tabular generative model evaluations~\cite{guzman-cordero2025exponential,kotelnikov2023tabddpm}.

\begin{table}[h]
\centering
\caption{Summary of datasets used in this study.}
\begin{tabular}{w{c}{3cm} w{c}{1.7cm} w{c}{1.7cm} w{c}{1.5cm} w{c}{2cm}}
\toprule
{Dataset} & {Records} & {\#Num.} & {\#Cat.} & {Source} \\
\midrule
UK Census & 104,267 & 1 & 14 & IPUMS \\
Canada Census & 32,149 & 4 & 21 & IPUMS \\
Fiji Census & 84,323 & 1 & 18 & IPUMS \\
Rwanda Census & 31,455 & 1 & 12 & IPUMS \\
Indonesia Census & 177,429 & 1 & 12 & IPUMS \\
Adult & 48,842 & 5 & 10 & UCI Repo \\
Churn & 10,000 & 4 & 7 & UCI Repo \\
\bottomrule
\end{tabular}
\label{tab:dataset}
\end{table}


\textbf{Network architecture. }We use the Transformer-based VAE encoder--decoder from TabSyn~\cite{zhang2024mixedtype}. All flow models use an MLP with four hidden layers of widths [1024, 2048, 2048, 1024], with SiLU activations and no batch normalisation. We encode time using a 512-dimensional sinusoidal embedding~\cite{nasution2025flowmatchingtabulardata}.

\textbf{Training setup. }We begin by preprocessing the data. We transform numerical variables using quantile transformer, and categorical variables using one-hot encoding. We then feed the preprocessed data to VAE. We follow $\beta$-VAE training setup from~\cite{zhang2024mixedtype} who trained the VAE for 4000 epochs. Using the trained VAE encoder, we encode the data into its latent space representation. Subsequently, we used the flow models we discussed in Section~\ref{sec:design_latent} to learn the latent distribution. We used both OT and VP paths for all algorithms; therefore, there are \textit{eight} models to compare. For each model, we used the model with the lowest training loss in 10,000 epochs, consistent with previous tabular synthesis work~\cite{zhang2024mixedtype,guzman-cordero2025exponential,nasution2025flowmatchingtabulardata}. This keeps the training budget consistent across latent flow variants.

\textbf{Sampling and integration.}
Unless stated otherwise, we sample with an ODE solver using Euler integration for 100 steps up to $t_{\mathrm{int}}=1$. For numerical stability, VP-based runs stop at $t_{\mathrm{int}}=1-10^{-5}$ to avoid the numerical issues when converting to a score near $t=1$~\footnote{This avoids division by $\sigma_t$ when converting to a score near $t=1$, primarily for SM, NM, and VFM.}. For OT sampling with NM and SM, we additionally clamp the start time to $t_{\mathrm{start}}=10^{-2}$ (see Appendix~\ref{ssec:why-ot-path} of the supplementary material for motivation). For SDE sampling and solver comparisons, we decouple the injected noise from the path variance to keep the diffusion term non-negligible. We use $g_t=\sigma_t^{\mathrm{VP}}$ on the OT path and $g_t=\sigma_t^{\mathrm{OT}}$ on the VP path.

\textbf{External comparison baselines.}
To understand the position of the latent flow models in generative models, we also compared our representative configurations with standard tabular synthesisers from major model families, following common baseline choices in recent tabular generative modelling studies~\cite{mueller2025continuous,xu2019modeling}. We compare against Bayesian Network~\cite{Ankan2024pgmpy}, Tabular Variational Autoencoder (TVAE) and Robust TVAE (RTVAE)~\cite{xu2019modeling,Akrami2022robust}, Conditional Tabular Generative Adversarial Networks (CTGAN)~\cite{xu2019modeling}, and Normalising Flows~\cite{durkan2019neural,qian2023synthcity}. Using the \texttt{synthcity} implementations~\cite{qian2023synthcity}, we keep algorithm-specific hyperparameters at their default values and only adjust shared training-budget settings where supported. These include batch size, maximum number of iterations, and network architecture when the implementation exposes directly configurable architecture settings. Further, we also include TabDDPM and TabSyn as modern diffusion-based references, with results taken from~\cite{nasution2025flowmatchingtabulardata} under the same evaluation protocol.


\subsection{Evaluation Metrics}
\label{sec:evaluation}
Following established practice in synthetic tabular data evaluation~\cite{taub2019the,elliot2023samples,Little2022comparing,patki2016SDV}, we focus primarily on two dimensions: analytical utility and disclosure risk. Our evaluation is motivated by a practical pre-release setting in which a data owner may select a synthetic tabular data generator before sharing data for external analysis or research access. In this setting, synthetic data is useful only if it preserves statistical properties required for downstream analysis, but it should not introduce unnecessary disclosure risk. 

\textbf{Utility (Preservation of Statistical Properties):}
Utility value ranges from 0 (no preservation) to 1 (perfect preservation), computed as:
\begin{equation}
\text{U} = \frac{1}{3} \left( \text{ROC}_{\text{univ}} + \text{ROC}_{\text{biv}} + \text{CIO} \right),
\end{equation}
where ROC measures frequency preservation (univariate and bivariate) and CIO measures statistical inference validity via confidence interval overlap.

\textbf{Disclosure Risk:}
Disclosure risk ranges from 0 to 1 and measures the normalised increase in empirical attribution success above a baseline adversary who only uses the original-data equivalence class structure. We use Targeted Correct Attribute Probability (TCAP) to evaluate a record-linkage-style attribute disclosure scenario, where an adversary knows a set of key variables and attempts to infer sensitive attributes using the synthetic data. The reported value should therefore be interpreted as an empirical disclosure risk diagnostic rather than a general privacy risk measure or release-safety guarantee. For example, $R=0.2$ does not mean that the dataset has a 20\% absolute privacy risk. Instead, it means that, relative to the baseline attribution uncertainty, the synthetic data accounts for 20\% of the remaining possible increase in empirical attribution success under the specified key and sensitive variables.

\textbf{Additional Diagnostic Metrics:} In addition to utility and risk, we also report Shape and Trend as complementary fidelity diagnostics that are widely used in synthetic data evaluation~\cite{zhang2024mixedtype,guzman-cordero2025exponential,mueller2025continuous}. Shape measures column-wise distributional agreement between real and synthetic data, while Trend measures how well pairwise relationships between variables are preserved. These metrics complement utility and disclosure risk by capturing marginal and relational fidelity in the generated tabular data. We report these metrics as means over 20 independently generated synthetic datasets. Additional and detailed evaluations, including $\alpha$-precision, $\beta$-recall, and Wasserstein distance, are reported in the appendix due to space constraints (Appendices~\ref{sec:evaluation_details} and~\ref{app:detail-res}).

\subsection{Configuration Study: Targets, Paths, Sampling Dynamics, and Solvers}
\label{sec:res-q123}

The paired OT--VP plots in Figure~\ref{fig:paired_ot_vp_all} and Table~\ref{tab:q123_summary} summarise how latent tabular flow models behave under three practical design choices: the \emph{learning target} (FM, SM, NM, VFM), the \emph{path} (OT vs.\ VP), and the \emph{sampling implementation} (ODE vs.\ SDE, and Euler vs.\ Midpoint)~\footnote{We used 100 integration steps. As a consequence, Euler and Euler--Maruyama use one model evaluation per step, while Midpoint uses two, so Midpoint has roughly twice the evaluation cost at the same step count.}. We report utility and disclosure risk as primary decision metrics, and use shape and trend as secondary fidelity checks. Because aggregate metric values can be affected by dataset-specific scale and structure, we also report per-dataset, rank-based, and pareto front summaries in Appendix~\ref{app:rank-pareto} of the Supplementary Material.

\begin{figure}[h]
    \centering
    \includegraphics[trim={0cm 0cm 0cm 0cm}, clip, width=\linewidth]{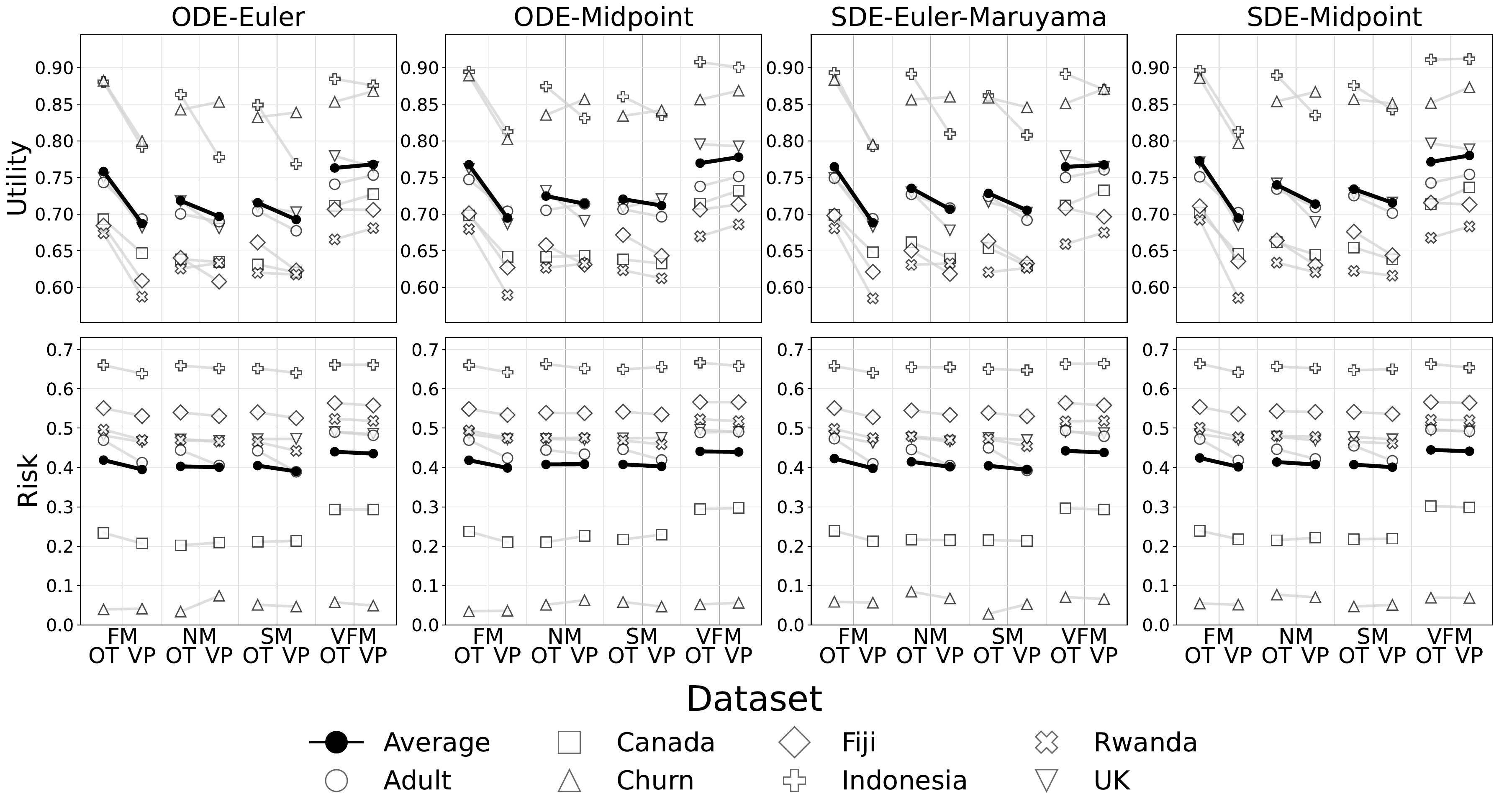}
    \caption{Paired OT vs.\ VP comparisons across learning targets (FM, NM, SM, VFM) and integrators. Each line connects the same dataset under OT and VP. Each hollowed shape corresponds to the datasets while the black, solid line represents the \textbf{average} values across all datasets. Top: Utility, bottom: disclosure risk. On average, FM, NM, and SM tend to favour the OT path in utility, whereas VFM is more competitive across OT and VP. Midpoint and SDE sampling shift the operating point, with effects that depend on the target and dataset.}
    \label{fig:paired_ot_vp_all}
\end{figure}


\begin{table}[t]
\centering
\caption{Aggregate Shape and Trend comparison addressing Q3 across learning targets, sampling dynamics, and numerical solvers. Values are averaged over 20 runs per dataset and then averaged across datasets. Lower values are better. Bold and underlined entries denote the best and second-best integrator setting for each target-path combination, respectively.}
\label{tab:q123_summary}
\begin{tabular}{w{c}{1.5cm} w{c}{2cm} w{c}{2cm} w{c}{2cm} w{c}{2cm} w{c}{2cm}} 
\toprule
Evaluation & Algorithm & ODE-Euler & ODE-Midpoint & SDE-EM & SDE-Midpoint\\
\midrule
\multirow{8}{*}{Shape (\%,$\downarrow$)} & FM-OT & 1.2329 & \underline{1.1796} & 1.1816 & \textbf{1.1537} \\
 & FM-VP & 1.9641 & \textbf{1.9260} & 1.9595 & \underline{1.9346} \\
 & NM-OT & 1.6587 & 1.5818 & \underline{1.3962} & \textbf{1.3484} \\
 & NM-VP & 1.9462 & \underline{1.5993} & 1.7573 & \textbf{1.5845} \\
 & SM-OT & 1.6847 & 1.5709 & \underline{1.4109} & \textbf{1.3656} \\
 & SM-VP & 1.8760 & \underline{1.6242} & 1.7145 & \textbf{1.5958} \\
 & VFM-OT & 1.3813 & \underline{1.3792} & \textbf{1.3703} & 1.3869 \\
 & VFM-VP & \textbf{1.2621} & 1.2935 & \underline{1.2705} & 1.2992 \\
\midrule
\multirow{8}{*}{Trend (\%,$\downarrow$)} & FM-OT & 2.6514 & \underline{2.5246} & 2.5604 & \textbf{2.4821} \\
 & FM-VP & 3.8475 & \textbf{3.7307} & 3.8586 & \underline{3.7479} \\
 & NM-OT & 3.2847 & 3.1623 & \underline{2.9751} & \textbf{2.8352} \\
 & NM-VP & 3.5545 & \underline{3.0452} & 3.3029 & \textbf{3.0073} \\
 & SM-OT & 3.3477 & 3.1546 & \underline{2.9753} & \textbf{2.8713} \\
 & SM-VP & 3.5654 & \underline{3.0739} & 3.3311 & \textbf{3.0319} \\
 & VFM-OT & \underline{2.6307} & \textbf{2.5351} & 2.6630 & 2.6504 \\
 & VFM-VP & 2.4862 & \textbf{2.4542} & 2.4962 & \underline{2.4555} \\
\bottomrule
\end{tabular}
\end{table}

\textbf{Q1 (Targets).}
Across datasets and integrators, FM-OT and VFM form a high-utility group, while NM and SM can be characterised as a risk-conservative but utility-limited group. FM typically generates synthetic data with strong utility and moderate disclosure risk. In contrast, NM and SM tend to produce lower-utility synthetic data, but their post-hoc disclosure risk evaluation indicates lower risk levels. VFM achieves high utility across integrator settings, although its disclosure risk profile is generally higher than NM and SM. This makes VFM more attractive when utility is prioritised, provided that the resulting disclosure risk remains within an acceptable threshold.

\textbf{Q2 (Paths: OT vs.\ VP).}
The comparison between interpolants or probability paths is target-dependent and mainly shows a consistent direction in the averaged paired comparisons. For FM, NM, and SM, switching from OT to VP generally decreases utility, but can also move the configuration toward a lower disclosure-risk regime, particularly for FM and SM. For VFM, VP tends to increase average utility relative to OT, while disclosure-risk changes are smaller than the corresponding utility shifts. Practically, OT is a more reliable default for FM when utility is prioritised, and it can also benefit NM and SM when used with a practical start-time clamp. By contrast, VFM is comparatively less sensitive to the path, with VP often slightly better.

\textbf{Q3 (Sampling dynamics and solver choice).}
Table~\ref{tab:q123_summary} shows that solver and sampling dynamics interact with the target--path choice. Midpoint integration often improves Shape and Trend errors over Euler at the same step count, although it requires roughly twice as many model evaluations. SDE-Midpoint is frequently the best or second-best setting, especially for NM and SM, but its advantage is not uniform: ODE sampling remains competitive for FM-VP and VFM. Figure~\ref{fig:paired_ot_vp_all} shows that utility and disclosure risk effects are more target- and dataset-dependent. Thus, stochastic sampling should be viewed as a calibration option rather than a universally better default, especially EM method.




\subsection{Integration Endpoint}
\label{sec:res-t-ode}

We vary the ODE integration endpoint $t_{\mathrm{int}} \in [0.6,1]$ to study the effect of early stopping during generation. Figure~\ref{fig:res_t_ode} reports averages across seven datasets as $t_{\mathrm{int}}$ increases from truncated to near-complete integration. Across all methods, utility increases with $t_{\mathrm{int}}$, accompanied by a simultaneous increase in disclosure risk. The path choice impacts truncation sensitivity. OT-based variants (FM-OT, VFM-OT) achieve high utility in the earlier stage ($t_{\mathrm{int}} \approx 0.7-0.8$) but also accumulate risk faster. In contrast, VP-based variants (FM-VP, SM, NM, VFM-VP) are more sensitive to truncation, requiring nearly full integration ($t_{\mathrm{int}} \approx 1.0$) to match the utility of OT paths. Among the methods, VFM (red lines) generally yields the highest utility envelope across the integration range, whereas SM and NM (orange/green dashed) lag slightly behind in the early stages. Dataset-specific results are reported in Appendix~\ref{sec:res-t-ode-detailed} of the supplementary material.

\begin{figure}[ht]
    \centering
    \includegraphics[trim={0cm 0cm 0cm 0cm}, clip, width=0.8\linewidth]{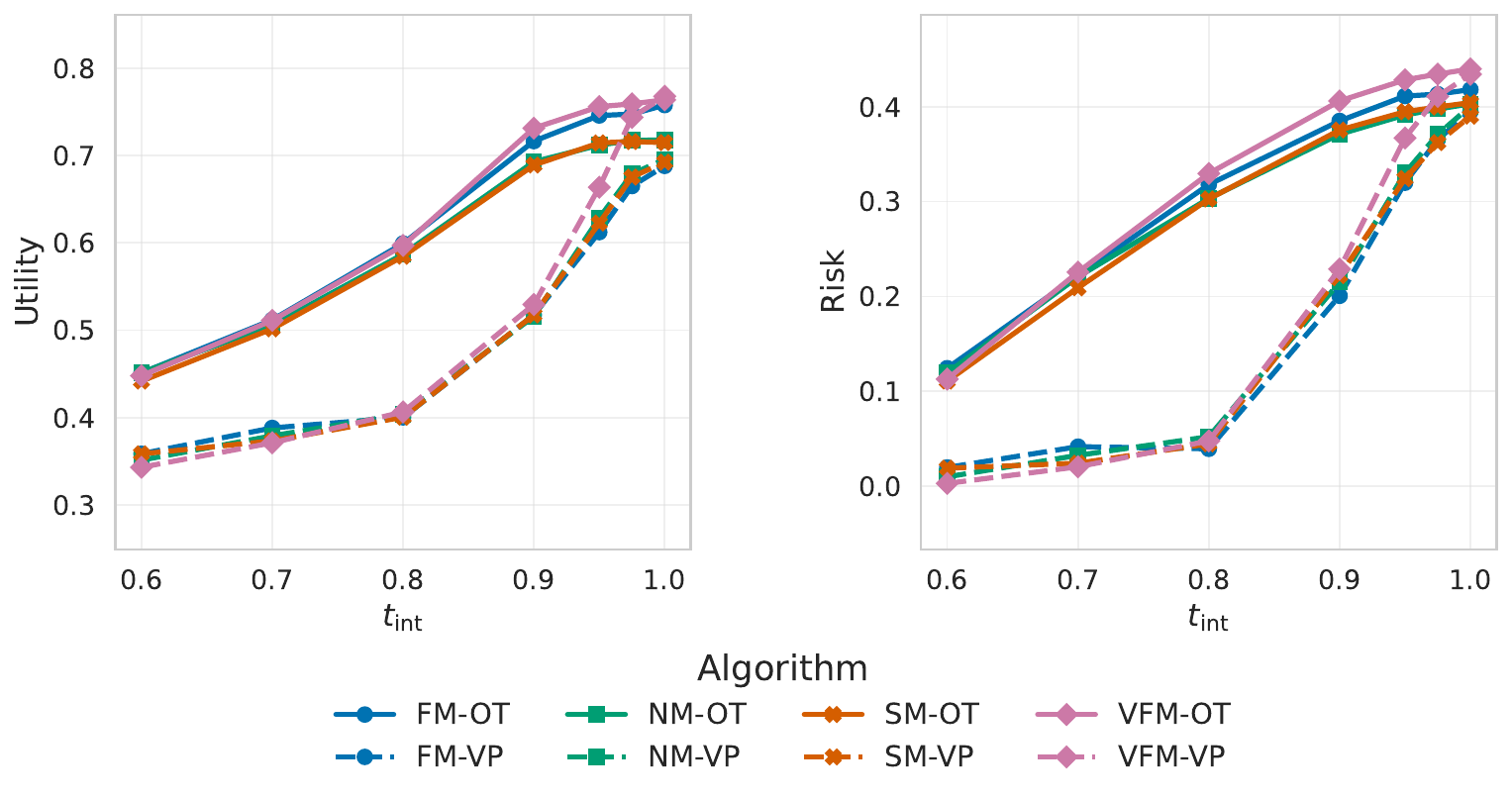}
    \caption{Average performance across seven datasets on integration time $t_{\mathrm{int}} \in [0.6, 1.0]$. The solid lines (OT) show consistently higher utility and risk at earlier stages compared to dashed lines (VP). The utility and risk increase monotonically as $t_{\mathrm{int}}\to 1$.}
    \label{fig:res_t_ode}
\end{figure}

\subsection{Integration Steps}
\label{sec:res-nfe}

We study sampling budget by varying the number of discrete steps during numerical integration. Figure~\ref{fig:res_nfe} reports average utility and disclosure risk across seven datasets for steps from 4 to 1024, with the vertical dashed line marking our default setting (100 steps). For interpretation, we group budgets into three regimes: \textbf{Phase I} (4--32 steps), \textbf{Phase II} (64--256 steps), and \textbf{Phase III} (512--1024 steps).

\begin{figure}[ht]
    \centering
    \includegraphics[trim={0cm 0cm 0cm 0cm}, clip,width=0.8\linewidth]{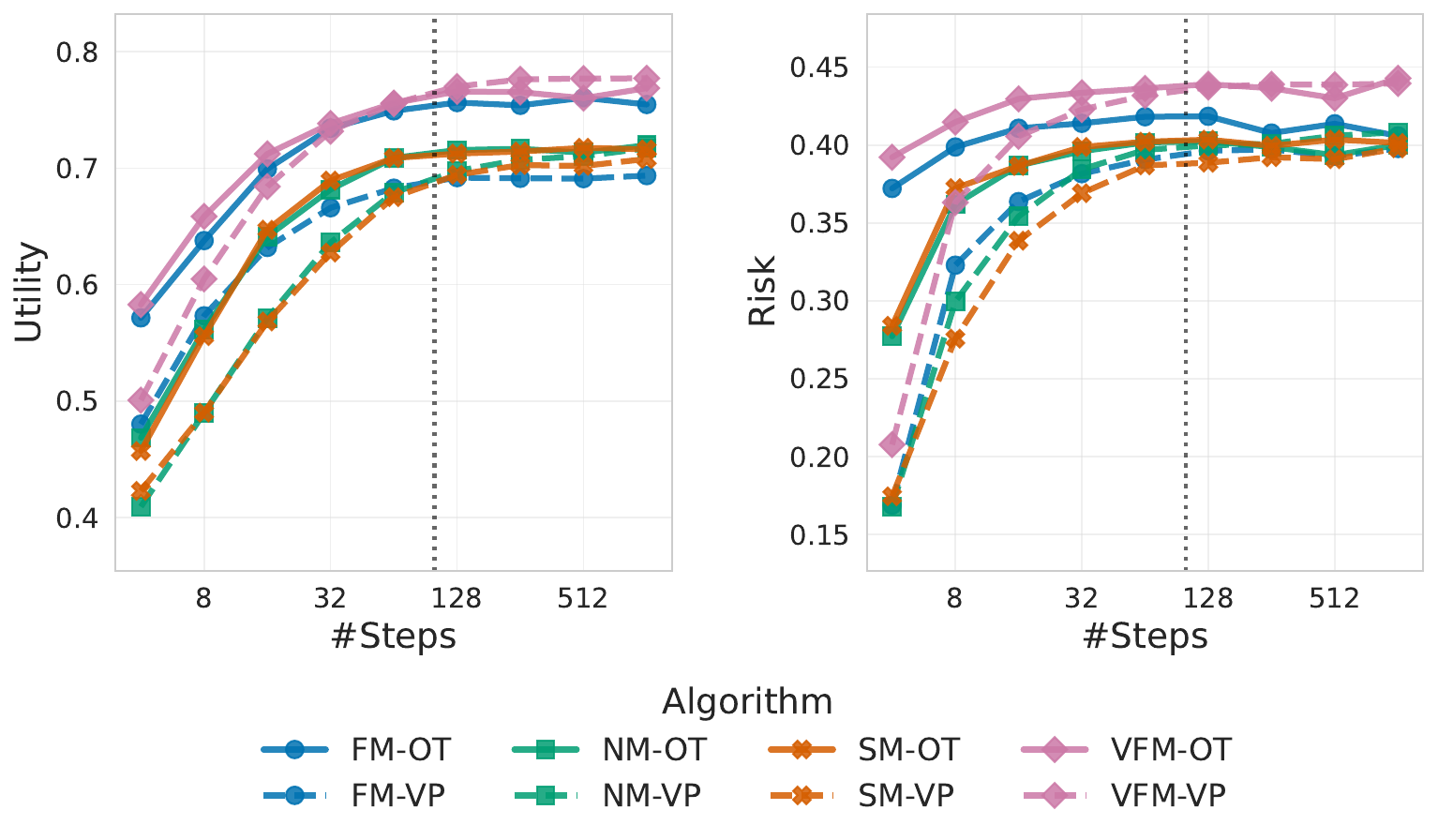}
    \caption{Average performance across seven datasets for varying steps ($2^n, n \in [2,10]$). The vertical line marks the baseline of 100 steps. OT paths (solid) achieve useful utility at lower budgets ($<32$ steps) compared to VP paths (dashed). Performance for all methods tends to converge and plateau beyond 100 steps.}
    \label{fig:res_nfe}
\end{figure}

Across methods, utility increases sharply in Phase I and then shows diminishing returns, while disclosure risk rises more steadily as budget increases. Differences between configurations are most pronounced under tight budgets: OT-based settings reach higher utility with fewer evaluations, whereas VP-based settings typically require substantially larger budgets to approach similar utility. Beyond the default range, performance gaps narrow and marginal utility gains become small while risk continues to increase. Overall, these trends support 100 steps as a reasonable default, with higher budgets best reserved for utility-first settings where additional cost and a modest increase in risk are acceptable. Dataset-specific results are reported in Appendix~\ref{sec:nfe-detailed} of the supplementary material.


\subsection{Comparison with Standard Tabular Synthesisers}
To contextualise the internal configuration study, we compare representative latent flow configurations with standard tabular synthesisers. Based on the internal results, we select FM-OT as a balanced latent flow configuration and VFM-VP as a utility-first configuration. Both are evaluated with the default ODE-Euler sampler, so that the external comparison reflects model-level differences.

Table~\ref{tab:external_baselines} contextualises the recommended latent flow configurations against standard tabular synthesisers. FM-OT and VFM-VP improve shape and trend errors over the non-diffusion baselines and achieve higher utility than the neural baselines, while remaining competitive with TabSyn. VFM-VP achieves the highest average utility and the lowest trend error, while FM-OT achieves the lowest shape error. The main trade-off is disclosure risk: lower-risk baselines such as RTVAE preserve less statistical structure. The comparison therefore supports viewing these methods as operating points on a utility--risk frontier rather than ranking them by a single metric.

\begin{table}[t]
\centering
\small
\setlength{\tabcolsep}{5pt}
\caption{Comparison with standard tabular synthesisers. Results are averaged across seven datasets. TabDDPM and TabSyn results are taken from~\cite{nasution2025flowmatchingtabulardata} under the same evaluation protocol. FM-OT and VFM-VP are representative latent flow configurations selected from the internal configuration study and evaluated with the default ODE-Euler sampler. Utility is higher-is-better, while disclosure risk, shape error, and trend error are lower-is-better.}
\label{tab:external_baselines}
\begin{tabular}{lcccc}
\toprule
\textbf{Model} & \textbf{Utility $\uparrow$} & \textbf{Risk $\downarrow$} & \textbf{Shape $\downarrow$} & \textbf{Trend $\downarrow$} \\
\midrule
Bayesian Network & 0.6275 & 0.3674 & 10.0305 & 17.0801\\
TVAE             & 0.5145 & 0.3392 & 11.4874 & 21.3613\\
RTVAE            & 0.4531 & \textbf{0.0563} & 18.4225 & 42.1489\\
CTGAN            & 0.5164 & 0.2484 & 14.8998 & 24.8076\\
Normalising Flow & 0.5149 & 0.2472 & 13.7603 & 20.8416\\
TabDDPM~\cite{nasution2025flowmatchingtabulardata}          & 0.5561 & 0.2620 & 17.0871 & 27.6257\\
TabSyn~\cite{nasution2025flowmatchingtabulardata}           & 0.7670 & 0.4314 & 1.2429 & 2.6557\\
\midrule
FM-OT            & 0.7582 & 0.4187 & \textbf{1.2329} & 2.6514\\
VFM-VP           & \textbf{0.7679} & 0.4352 & 1.2621 & \textbf{2.4862}\\
\bottomrule
\end{tabular}
\end{table}

\section{Discussions}

This section interprets the empirical results and derives practical guidance. Section~\ref{sec:discussion} synthesises the key findings across learning targets, probability paths, and sampling numerics, and Section~\ref{sec:practical_workflow} translates these findings into practical configuration guidance with actionable defaults for practitioners.

\subsection{Discussion of Findings}
\label{sec:discussion}

The performance of latent tabular synthesis is governed by three interacting choices: the learning objective, the probability path, and the sampling numerics.

\textbf{Targets set the utility--risk regime (\textbf{Q1}).}
FM and VFM occupy the higher-utility regime, while NM and SM are more conservative, trading lower utility for reduced disclosure risk. FM-OT emerges as a balanced choice. Meanwhile, VFM is more suitable when utility is prioritised within an acceptable risk threshold, but it generally has a less conservative disclosure risk profile than NM and SM.

\textbf{Path effects are target-dependent and asymmetric (\textbf{Q2}).}
OT and VP are not interchangeable across targets. Switching to VP tends to reduce utility for FM, NM, and SM, while often moving FM and SM toward more conservative disclosure-risk positions. VFM behaves differently: VP often improves utility relative to OT at broadly comparable risk. OT is therefore a more reliable default for FM, NM, and SM, whereas VFM benefits more reliably from VP in utility-first settings.

\textbf{Sampling numerics adjust the operating point (\textbf{Q3-Q5}).}
The results for Q3--Q5 show that sampling configurations should be treated as calibration decisions rather than fixed implementation details. For Q3, SDE sampling effects are more target- and dataset-dependent. Midpoint can improve fidelity over Euler at roughly doubled evaluation cost, whereas Euler remains a sensible default under tighter compute budgets. For Q4, the integration endpoint $t_{\mathrm{int}}$ controls an explicit utility--risk relationship: later stopping generally increases utility but also raises disclosure risk, with OT paths reaching higher utility earlier than VP paths. For Q5, increasing the number of integration steps improves utility most strongly under small budgets, but the gains diminish beyond the mid-range while disclosure risk can continue to increase. Together, these results support a risk-constrained configuration strategy: choose the target and path first, then tune solver, dynamics, endpoint, and step budget to meet the intended utility, fidelity, risk, and compute budget.

\subsection{Practical Configuration Guidance}
\label{sec:practical_workflow}

Figure~\ref{fig:workflow} summarises how the empirical findings can guide configuration choices before any domain-specific disclosure review. The workflow separates training-time design from sampling-time calibration. Based on our results, \textbf{FM-OT} is a balanced default, \textbf{VFM} is preferable when utility is prioritised, and \textbf{SM} or \textbf{NM} can be used for more conservative settings. For SM and NM under the OT path, one can avoid starting exactly at $t=0$ by using a small start-time clamp.

\begin{figure}[h]
\centering
\includegraphics[trim={0.5cm 0.5cm 0.5cm 0.5cm}, clip,width=0.8\linewidth]{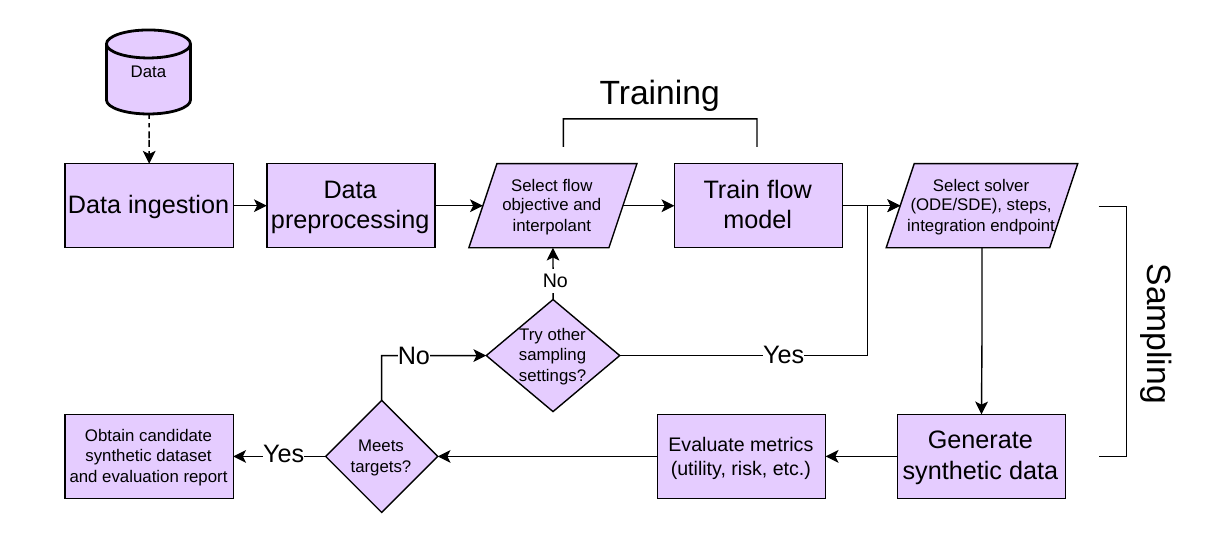}
\caption{Configuration workflow for latent flow tabular synthesis. A training-time design is selected first, then sampling settings such as ODE/SDE dynamics, solver, integration steps, and $t_{\mathrm{int}}$ are calibrated against pre-defined utility and disclosure risk constraints. The workflow produces a candidate synthetic dataset and evaluation report for subsequent domain-specific review.}
\label{fig:workflow}
\end{figure}

After training, sampling choices adjust the operating point under a fixed compute budget. Euler is the simplest default and uses one model evaluation per step, while Midpoint can improve fidelity at higher evaluation cost. SDE sampling often improves shape and trend errors in our aggregate results, but should be used only when the resulting disclosure risk remains acceptable. Integration steps and $t_{\mathrm{int}}$ provide additional controls over the utility and disclosure risk. In practice, users first impose a disclosure risk constraint, then select the most useful and faithful configuration within that constraint. Acceptable values of disclosure risk are therefore domain-dependent and should be treated as release-context constraints rather than universal safety thresholds.

\section{Concluding Remarks}
\label{sec:conclusion}

We presented a controlled empirical study of latent-space tabular flow models, showing that theoretically related learning targets and sampling dynamics can lead to different empirical behaviour under finite-step budgets. The utility--risk relationship is primarily governed by target and path choice: VFM-VP is a utility-first configuration, FM-OT provides a strong high-utility default with lower disclosure risk than VFM, and lower-risk regimes are generally obtained by accepting reduced utility, especially under VP-based or NM/SM configurations. Sampling choices then adjust the operating point: SDE sampling, particularly Midpoint can improve fidelity at higher evaluation cost, and OT paths enable compute savings or risk-aware synthesis through early stopping. Although the effect is target- and dataset-dependent, these findings support pre-release configuration and evaluation of latent flow tabular synthesis under domain-specific utility, risk, and compute constraints.

While the study provides practical configuration guidance, we also acknowledge several limitations. First, we set the VAE architecture to isolate the impact of flow dynamics as fixed. Second, utility--risk trajectories are dataset-dependent, so practitioners should perform target-domain validation before any release decision. Third, we restrict the analysis to fixed-step numerical solvers. Future work should evaluate broader architectures, adaptive ODE/SDE integrators, additional tabular domains, and mixed-effect modelling on the evaluation for more comprehensive insights.

\begin{credits}
\subsubsection{\ackname}
B.I. Nasution is supported by the Indonesia Endowment Fund for Education (LPDP) scholarship. B.I. Nasution thanks Mark Elliot and Richard Allmendinger for their supervision and valuable feedback on this work.
\subsubsection{\discintname} The author declares that there are no competing interests.
\subsubsection{LLM Usage Statement.} During the preparation of this manuscript, the author used large language models (LLMs) primarily for editing and improving the readability of the text. All analyses and claims are entirely the author's original work and responsibility.

\end{credits}

%
%
%
%

\bibliography{main}
\bibliographystyle{splncs04}

\pagenumbering{gobble}
\appendix

\newpage

\section*{Supplementary Materials for ``Understanding Latent Flow Models for Tabular Data Synthesis: Targets, Paths, and Sampling''}

\section{Training Algorithm of Latent Flow Model}
\label{sec:algorithm}

    \begin{algorithm}[ht]
        \caption{An iteration of latent flow model training}
        \label{alg:TabSynFlow-t}
        \textbf{Input:} Latent representation from VAE $\mathcal{Z}_1$, interpolation function $\alpha_t$ and $\sigma_t$, latent generative network $\mathrm{NN}(.)$
        \begin{algorithmic}[1]
            \State $z_1\sim p_{\mathcal{Z}}$
            \State $z_0 \sim \mathcal{N}(0, I)$
            \State $t \sim \mathrm{Uniform}(0,1)$
            \State $z_t = \alpha_t z_1 + \sigma_t z_0$
            \State Calculate target \Comment{FM, SM, NM or VFM}
            \State Predict $\text{F}_t^{\theta}(z_t) = \mathrm{NN}(z_t,t)$
            \State Calculate $\mathcal{L}_{\Box}(\theta)$
            \State Update $\theta$
        \end{algorithmic}
    \end{algorithm}

\section{Proofs and Theoretical Motivations}

\subsection{Velocity-Score Equivalence (Proposition~\ref{prop:vel_score})}
\label{ssec:proof-velscore}

We restate Proposition~\ref{prop:vel_score}.
For a Gaussian probability path with interpolants $\alpha_t$ and $\sigma_t$, the conditional velocity field $u_t(z_t|z_1)$ and the conditional score function $\nabla_{z_t} \log p_t(z_t|z_1)$ are related by:
\[u_t(z_t|z_1) = \left(\sigma_t^2 \frac{\dot{\alpha}_t}{\alpha_t} - \dot{\sigma}_t \sigma_t\right) \nabla_{z_t} \log p_t(z_t|z_1) + \frac{\dot{\alpha}_t}{\alpha_t} z_t.\]
\textit{Proof}. Recall the velocity equation in Equation~\eqref{eq:cond_velocity_z}.
\[u_t(z_t\mid z_1)=\dot{\alpha}_t z_1+\dot{\sigma}_t z_0\]
By solving $z_0$ from $z_t=\alpha_t z_1+\sigma_t z_0$, we get the velocity in terms of $z_1$ and $z_t$.
\begin{equation}
    \begin{split}
        u_t(z_t\mid z_1) &= \dot{\alpha}_t z_1+\dot{\sigma}_t z_0 \\
        &= \dot{\alpha}_t z_1 + \dot{\sigma}_t \Big( \frac{z_t-\alpha_t z_1}{\sigma_t} \Big) \\
    \end{split}
    \label{eq:vel_ext}
\end{equation}
Subsequently, recall the conditional score and solve for $z_1$.
\begin{equation}
    \begin{split}
          \nabla_{z_t} \log p_t(z_t|z_1) &= -\frac{z_t-\alpha_t z_1}{\sigma_t^2}  \\
          \alpha_t z_1 &= z_t + \sigma_t^2 \nabla_{z_t} \log p_t(z_t|z_1) \\
          z_1 &= \frac{z_t + \sigma_t^2 \nabla_{z_t} \log p_t(z_t|z_1)}{\alpha_t}
    \end{split}
\end{equation}
Plugging $z_1$ and $\alpha_t z_1$ into Equation~\eqref{eq:vel_ext} we get
\begin{equation}
    \begin{split}
        u_t(z_t\mid z_1) &= \dot{\alpha}_t \frac{z_t + \sigma_t^2 \nabla_{z_t} \log p_t(z_t|z_1)}{\alpha_t} + \dot{\sigma}_t \Big( \frac{z_t - z_t - \sigma_t^2 \nabla_{z_t} \log p_t(z_t|z_1)}{\sigma_t} \Big) \\
        &=\frac{\dot{\alpha}_t }{\alpha_t} z_t + \frac{\dot{\alpha}_t }{\alpha_t}\sigma_t^2 \nabla_{z_t} \log p_t(z_t|z_1) - \dot{\sigma}_t\sigma_t \nabla_{z_t} \log p_t(z_t|z_1) \\
        &= \Big( \sigma_t^2\frac{\dot{\alpha}_t }{\alpha_t} - \dot{\sigma}_t\sigma_t \Big) \nabla_{z_t} \log p_t(z_t|z_1) + \frac{\dot{\alpha}_t }{\alpha_t} z_t
    \end{split}
\end{equation}
which is exactly Equation~\ref{eq:vel_score_relation}. This completes the proof.


\subsection{Why is It Hard Generating Sample using OT Path in NM and SM?}
\label{ssec:why-ot-path}

Score matching (SM) and noise matching (NM) do not directly parameterise the velocity field used by the sampling ODE. Instead, at generation time we must convert the learnt score (or noise, via a score) into velocity using Equation~\eqref{eq:velocity_from_score}. Recall that the velocity can be written as
\[v_t^\theta(\tilde z_t)=
\Big(\sigma_t^2 \frac{\dot{\alpha}_t}{\alpha_t}-\dot{\sigma}_t \sigma_t\Big)s_t^\theta(\tilde z_t)
+\frac{\dot{\alpha}_t}{\alpha_t}\tilde z_t.\]
For the OT path, we have $\alpha_t=t$ and $\sigma_t=1-(1-\sigma_{\min})t$, hence $\dot{\alpha}_t=1$ and $\dot{\sigma}_t=1-\sigma_{\min}$.
Substituting into $v_t^\theta(\tilde z_t)$ yields
\begin{equation}
\begin{split}
v_t^\theta(\tilde z_t)
&=
\Big((1-(1-\sigma_{\min})t)^2 \frac{1}{t} - (1-\sigma_{\min})(1-(1-\sigma_{\min})t)\Big)s_t^\theta(\tilde z_t)
+\frac{1}{t}\tilde z_t \\
&=
\Big(\frac{(1-(1-\sigma_{\min})t)^2}{t} + (1-\sigma_{\min})(1-(1-\sigma_{\min})t)\Big)s_t^\theta(\tilde z_t)
+\frac{1}{t}\tilde z_t
\end{split}
\label{eq:velocity_from_score_ot}
\end{equation}
The second term of Equation~\eqref{eq:velocity_from_score_ot} contains the factor $1/t$, which makes the term $\frac{1}{t}\tilde z_t$ highly unstable when $t\to 0$.
Since $\tilde z_t$ is $O(1)$ in magnitude near the start of sampling (it is close to the base Gaussian initialisation),
the product $\frac{1}{t}\tilde z_t$ can become arbitrarily large for small $t$.
This creates a stiff ODE near the initial time: even moderate step sizes may yield unstable updates or require excessively
small steps to remain stable.

More broadly, SM/NM sampling often involves divisions by path coefficients (directly or indirectly):
for OT, $\alpha_t=t$ vanishes at $t\to 0$, while $\sigma_t=1-t$ vanishes at $t\to 1$.
As a consequence, conversions between velocity, score, and noise can exhibit endpoint instabilities under OT,
making OT substantially harder to use robustly for SM and NM than for velocity-parameterised FM.

\textbf{Practical workaround.}
A common mitigation is to avoid integrating exactly from $t=0$ by starting from a small $t_{\mathrm{start}}>0$
(i.e., $t\in[t_{\mathrm{start}},1]$) to prevent the $1/t$ blow-up.
In our simulations, we set $t_{\mathrm{start}}=0.01$ rather than starting at $t=0$.



\newpage

\section{Experiment details}
\label{sec:exp_details}

Table~\ref{tab:hyperparams_train} displays hyperparameters used in out latent flow models, while Table~\ref{tab:baseline_hyperparams} displays hyperparameters used in the external baseline algorithms. Note that to make the comparison fair, for deep generative models we set similar epochs, batch size and also adjust the number of parameters to be close with our generator networks, as seen in Table~\ref{tab:baseline_parameter_count}. We mainly run the experiment on HPC with GPU NVIDIA L40S. However, considering the number of parameters, this experiment can be re-run in smaller GPU such as RTX family.

\begin{table}[ht]
\centering
\caption{Training hyperparameters for Latent flow models.}
\begin{tabular}{ll p{3cm} p{4cm}}
\toprule
Model & Hyperparameter & {Value} & {Notes} \\
\midrule
\multirow{5}{*}{VAE} & Batch size & 4096 & Standard for tabular data \\
 & Optimiser & Adam & $\beta_1 = 0.9$, $\beta_2 = 0.999$ \\
 & Learning rate & $10^{-3}$ & Reduce LR on plateau \\
 & Max epochs & 4000 & - \\ 
 & $\beta$ schedule & Adaptive & Gradual KL weight increase \\ \midrule
\multirow{4}{*}{Transformer} & Token dimension & 4 & - \\
 & Attention head & 1 & - \\
 & Number of Layers & 2 & - \\
 & Factor & 32 & - \\ \midrule
\multirow{5}{*}{FM (training)} & Batch size & 4096 & Standard for tabular data \\
 & Optimiser & Adam & $\beta_1 = 0.9$, $\beta_2 = 0.999$ \\
 & Learning rate & $10^{-3}$ & Reduce LR on plateau \\
 & Max epochs & 10000 & Early stopping on training loss \\ 
 & Interpolant & OT, VP & - \\
 \midrule
\multirow{2}{*}{FM (network)}  & MLP layer & [1024, 2048, 2048, 1024] & SiLU activation \\
 & Time embedding & [512, 512] & positional embedding, SiLU activation \\ \midrule
\multirow{4}{*}{FM (sampling)} & Solver & Euler (default), midpoint & For SDE also uses Euler-Maruyama and Midpoint \\
 & \#Integration steps & 100 (default), $2^n, n=[2,10]$ & - \\
 & $t_{\mathrm{int}}$ & $[0.6,1.0]$ & Default: $1.0$; VP (except FM): $1-10^{-5}$; NM/SM-OT: $t_{\mathrm{start}}=10^{-2}$ \\
 & $g_t$ & OT, VP & Opposite with $\sigma_t$ used when training \\
 
\bottomrule
\end{tabular}
\label{tab:hyperparams_train}
\end{table}

\begin{table}[ht]
\centering
\caption{Training settings for baseline deep generative models. All baselines use their default implementation settings unless stated otherwise. We adjust only common training-budget parameters to align the comparison with the present study. The Bayesian Network baseline uses the default configuration.}
\begin{tabular}{ll p{3cm} p{4cm}}
\toprule
Model & Hyperparameter & Value & Notes \\
\midrule
\multirow{6}{*}{Common} 
 & Batch size & 4096 & Matched to the latent flow experiments \\
 & Learning rate & $10^{-3}$ & Used where the implementation exposes this option \\
 & Max epochs & 10000 & Matched training budget where supported \\ 
\midrule
\multirow{2}{*}{(R)TVAE} 
 & Network size & [768, 768, 768] & MLP with comparable encoder and decoder capacity \\
\midrule
\multirow{2}{*}{CTGAN} 
 & $G$ size & [640, 640, 640] & ResNet generator \\
 & $D$ size & [640, 640, 640] & MLP discriminator \\
\midrule
\multirow{2}{*}{Normalising Flow} 
 & Network size & [1280, 1280] & MLP transformation network \\
\bottomrule
\end{tabular}
\label{tab:baseline_hyperparams}
\end{table}

\begin{table}[ht]
    \centering
    \caption{Approximate number of trainable network parameters for deep generative models on the Canada dataset. Parameter counts depend on the dataset dimensionality after preprocessing. For TabSyn and the latent flow models, we report the VAE component and the generative model component separately.}
    \begin{tabular}{lc}
    \toprule
        Model & \#Parameters \\
    \midrule
        (R)TVAE & 12,476,391 \\
        CTGAN & 11,502,116 \\
        Normalising Flow & 17,075,386 \\
        TabDDPM & 11,900,078 \\
        TabSyn & 9,118 + 10,698,852 \\
        Latent Flow Models & 9,118 + 10,698,852 \\
    \bottomrule
    \end{tabular}
    \label{tab:baseline_parameter_count}
\end{table}

\clearpage

\section{Evaluation Details}
\label{sec:evaluation_details}

We evaluate the quality of synthetic data along three complementary dimensions: statistical utility~\ref{sec:utility_details}, disclosure risk~\ref{sec:privacy_details}, and additional diagnostic measures~\ref{sec:eval:diagnostics}. Statistical utility assesses how well the synthetic data preserve the marginal, joint, and inferential properties of the original data. Disclosure risk evaluates the extent to which synthetic records may increase an adversary's ability to infer sensitive information through record linkage. Additional diagnostics provide supplementary evidence on distributional fidelity and coverage. Following~\cite{nasution2025flowmatchingtabulardata}, we use the same designation of key variables and target variable throughout the evaluation, so that utility and risk are assessed under a consistent analytical setup.

\subsection{Statistical Utility Assessment}
\label{sec:utility_details}

\subsubsection{Ratio of Counts (ROC).}

Frequency-based analysis captures how well synthetic data mirror the marginal and joint distributions of the original dataset. For any contingency table formed by grouping records by categorical variables or their combinations, we compute per-cell agreement:

\begin{equation}
\text{ROC}_{\text{cell}} = \frac{\min(n_{\text{orig}}, n_{\text{synth}})}{\max(n_{\text{orig}}, n_{\text{synth}})}
\end{equation}

This formula yields 1.0 when the counts are identical and 0.0 when completely mismatched. We compute two variants:

\begin{itemize}
\item Univariate ROC ($\text{ROC}_\text{univ}$): Applied to each variable independently, measuring how well marginal frequencies are preserved
\item Bivariate ROC ($\text{ROC}_\text{biv}$): Applied to two-way cross-tabulations, measuring pairwise association preservation
\end{itemize}

\subsubsection{Confidence Interval Overlap (CIO).}

Statistical inference validity is assessed by comparing regression models fitted on original and synthetic data. For each regression coefficient, we extract 95\% confidence intervals and measure their degree of overlap:

\begin{equation}
\text{CIO}_j = 0.5 \cdot \left(\frac{\text{overlap}}{u_{\text{orig}} - l_{\text{orig}}} + \frac{\text{overlap}}{u_{\text{synth}} - l_{\text{synth}}}\right)
\end{equation}

where $l$ and $u$ denote the lower and upper interval bounds, and the overlap is equal to \[\min(u_{\text{orig}}, u_{\text{synth}}) - \max(l_{\text{orig}}, l_{\text{synth}}).\]

A CIO of 1.0 indicates identical confidence intervals; While $\leq0$ indicates that there is no intersection, therefore the value is truncated to 0 in practice~\cite{nasution2025flowmatchingtabulardata}. We fit models for categorical targets using logistic regression and continuous targets using linear regression. The dataset-level CIO is averaged across all coefficient estimates and target variables.

\subsubsection{Composite utility measure.}

The utility values combine both frequency preservation and inference preservation:

\begin{equation}
\text{U} = \frac{\text{ROC}_{\text{univ}} + \text{ROC}_{\text{biv}} + \text{CIO}}{3}
\end{equation}

Utility ranges 0-1, where 1 indicates perfect statistical utility and 0 indicates complete failure to preserve analytical properties.

\subsection{Disclosure Risk}
\label{sec:privacy_details}

We measure disclosure risk through a record linkage scenario in which an adversary knows a set of quasi-identifiers (key variables, $K$) and attempts to infer a sensitive attribute $T$ by using the synthetic data as a reference~\cite{Little2022comparing}. The metric evaluates the additional attribution risk introduced by synthetic data relative to the baseline uncertainty already present in the original data.

\subsubsection{Targeted Correct Attribution Probability (TCAP).} 
For each synthetic record $j$ with key variables $K'_j$ and sensitive value $T'_j$, the TCAP measures the probability that $T'_j$ is the correct sensitive value for the equivalence class defined by $K'_j$ in the original data~\cite{Little2024synthetic}. It is defined as
\begin{equation}
\label{eq:tcap_paper2}
\mathrm{TCAP}'_j
=
\mathbb{P}_{\mathcal{D}}(T'_j \mid K'_j)
=
\frac{
\sum_{i=1}^{n}
\mathbb{I}[T_i = T'_j,\ K_i = K'_j]
}{
\sum_{i=1}^{n}
\mathbb{I}[K_i = K'_j]
},
\end{equation}
where $\mathcal{D}$ denotes the original dataset and $n$ its size. Intuitively, $\mathrm{TCAP}'_j$ is the empirical probability that the sensitive value appearing in the synthetic record matches the true sensitive value in the original-data equivalence class with the same key variables.

\subsubsection{Within-Equivalence Class Attribute Probability (WEAP).} WEAP acts as a baseline, which measures the probability of an adversary who does not use the synthetic data and instead can only guess within the corresponding equivalence class in the original data~\cite{nasution2025flowmatchingtabulardata}.
\begin{equation}
\label{eq:weap_paper2}
\mathrm{WEAP}_j
=
\mathbb{P}_{\mathcal{D}}(T_j \mid K_j)
=
\frac{
\sum_{i=1}^{n}
\mathbb{I}[T_i = T_j,\ K_i = K_j]
}{
\sum_{i=1}^{n}
\mathbb{I}[K_i = K_j]
}.
\end{equation}

Although this has close empirical form as TCAP, the interpretation is different. $\mathrm{WEAP}_j$ represents the baseline probability of success from a random draw within the original-data equivalence class, that is, the level of attribution risk that already exists without access to synthetic data. 

\subsubsection{Risk Measurement.} By using those two components, we can measure the disclosure risk that is defined as the normalised increase in adversarial success above this baseline:
\begin{equation}
\label{eq:risk_paper2}
\text{R}_j
=
\max\left\{
0,\
\frac{\mathrm{TCAP}'_j - \mathrm{WEAP}_j}{1 - \mathrm{WEAP}_j}
\right\}.
\end{equation}

A value near $0$ indicates that the synthetic data provides little or no additional disclosure risk beyond the baseline in the original data, whereas a value near $1$ indicates a substantial increase in attribution risk. Negative values, corresponding to cases where the adversary performs worse than the baseline, are truncated to zero in practice~\cite{nasution2025flowmatchingtabulardata,elliot2023samples}. At the dataset level, disclosure risk is computed as the mean of $R_j$ over all synthetic records retained by the threshold criterion $\tau$. By construction, disclosure risk measures empirical linkage probability rather than worst-case theoretical scenarios, aligning with statistical agencies' practical risk quantification needs.

\subsection{Additional Evaluation Metrics}
\label{sec:eval:diagnostics}
In addition to utility and disclosure risk, we report auxiliary diagnostics to characterise distributional fidelity and detectability. Unless stated otherwise, lower values indicate closer agreement between real and synthetic data. All diagnostics are computed on the same preprocessed representation used by the data loader (continuous normalisation and categorical encoding). We used SDV~\cite{patki2016SDV} library for shape and trend, and Synthcity~\cite{qian2023synthcity} for $\alpha$-precision, $\beta$-recall, and Wasserstein distance.

\subsubsection{Column-wise fidelity (Shape).}
We summarise marginal discrepancies using a Shape score that averages (i) the Kolmogorov--Smirnov statistic over continuous columns and (ii) total variation distance over categorical columns.

\subsubsection{Pairwise structure (Trend).}
We summarise pairwise dependence discrepancies using a Trend score, combining absolute Pearson correlation differences for continuous--continuous pairs and a contingency-table discrepancy for categorical--categorical pairs, averaged across pairs.

\subsubsection{$\alpha$-precision and $\beta$-recall.}
We report $\alpha$-precision and $\beta$-recall as complementary measures of sample-wise fidelity and coverage. Intuitively, $\alpha$-precision measures the fraction of synthetic samples that fall within a high-density region of the real distribution, while $\beta$-recall measures the fraction of real samples covered by a high-density region of the synthetic distribution. We compute these metrics following the referenced procedure based on estimating high-density sets from samples.

\subsubsection{Distributional discrepancy (WD).}
We compute Wasserstein distance between real and synthetic samples in the encoded feature space. Considering memory constraint on large datasets, we estimate WD using $K$ batches of size $n=5{,}000$ sampled from both datasets and report the mean across batches.

\subsubsection{Pareto and rank-based summaries}

We include two additional summaries to complement the aggregate results in the main text. First, we report a Pareto front to visualise the trade-off between synthetic data utility and disclosure risk. A configuration is considered Pareto-efficient if no other configuration achieves both higher utility and lower disclosure risk. This summary is intended as a visual aid for identifying non-dominated operating points, rather than as a single model-selection rule.

Second, we report rank-based summaries across datasets. For each dataset, configurations are ranked separately for each metric. Utility is ranked in descending order, while disclosure risk is ranked in ascending order. We then report the mean and standard deviation of ranks across datasets. This provides a scale-invariant descriptive summary and reduces the sensitivity of comparisons to dataset-specific metric ranges. We do not combine the ranks or Pareto summaries into a single overall score, since the preferred trade-off between utility and disclosure risk depends on the intended application and release context.

\clearpage
\newpage
\section{Detailed Experimental Results}
\label{app:detail-res}

\subsection{Pareto Front and Rank Analysis}
\label{app:rank-pareto}

To extend the results on Figure~\ref{fig:paired_ot_vp_all}, we provided pareto front and rank analysis on the algorithms that we benchmarked. We focused the analysis on the primary metrics, namely the utility and risk.

\subsubsection{Correlation checks on evaluation results.}

To avoid comparing metrics on different numerical scales, we compute correlations at the rank level rather than using raw metric values. For each metric, configurations are first ranked according to the relevant performance direction: utility, alpha precision, and beta recall are ranked such that higher values are more favourable, while disclosure risk, Shape error, Trend error, and Wasserstein distance are ranked such that lower values are more favourable. We then compute Spearman correlations between these rank orderings.

The results show that utility rankings are strongly aligned with most fidelity-related rankings, including Shape, Trend, beta recall, and Wasserstein distance. Alpha precision shows a weaker alignment, suggesting that it captures a somewhat different diagnostic aspect or is less discriminative in this setting. Disclosure-risk rankings are negatively associated with utility and most fidelity-related rankings, indicating a clear utility--risk trade-off: configurations ranked highly for utility tend to be ranked less favourably under disclosure risk. This supports focusing the rank-based summary on the two primary model-selection axes, utility and disclosure risk, while retaining the remaining diagnostic metrics as detailed supplementary results.

\begin{table}[]
    \centering
    \caption{Lower triangle matrix of average Rank-level Spearman correlations between evaluation criteria across datasets. Metric directions are harmonised before ranking so that each rank reflects the relevant performance direction. Values are reported as mean $\pm$ standard deviation across datasets. Positive values indicate aligned rank orderings, while negative values indicate opposing rank orderings.}
\begin{tabular}{c|c|c|c|c|c|c}
\toprule
 & Utility & Risk & Shape & Trend & $\alpha$-pr. & $\beta$-re. \\
\midrule

Utility & - & - & - & - & - & -  \\
Risk & -0.74 ± 0.27 & - & - & - & - & - \\
Shape & 0.84 ± 0.08 & -0.64 ± 0.22 & - & - & - & -  \\
Trend & 0.83 ± 0.14 & -0.72 ± 0.27 & 0.88 ± 0.10 & - & - & - \\
$\alpha$-pr. & 0.41 ± 0.19 & -0.33 ± 0.20 & 0.52 ± 0.20 & 0.42 ± 0.26 & - & -  \\
$\beta$-re. & 0.77 ± 0.26 & -0.79 ± 0.28 & 0.67 ± 0.24 & 0.67 ± 0.37 & 0.49 ± 0.20 & - \\
WD & 0.82 ± 0.16 & -0.80 ± 0.24 & 0.71 ± 0.12 & 0.76 ± 0.20 & 0.46 ± 0.29 & 0.84 ± 0.17  \\
\bottomrule
\end{tabular}

    \label{tab:flow-model-rank-corr}
\end{table}

\subsubsection{Pareto and rank-based utility--risk summaries}

We provide two complementary summaries of the utility--disclosure risk trade-off. The Pareto-style plots visualise non-dominated configurations within each dataset, where a configuration is included in Pareto set if no other configuration achieves both higher utility and lower disclosure risk. These plots are intended as visual aids for interpreting operating regimes rather than as a single model-selection rule.

\begin{figure}
    \centering
    \includegraphics[width=\linewidth]{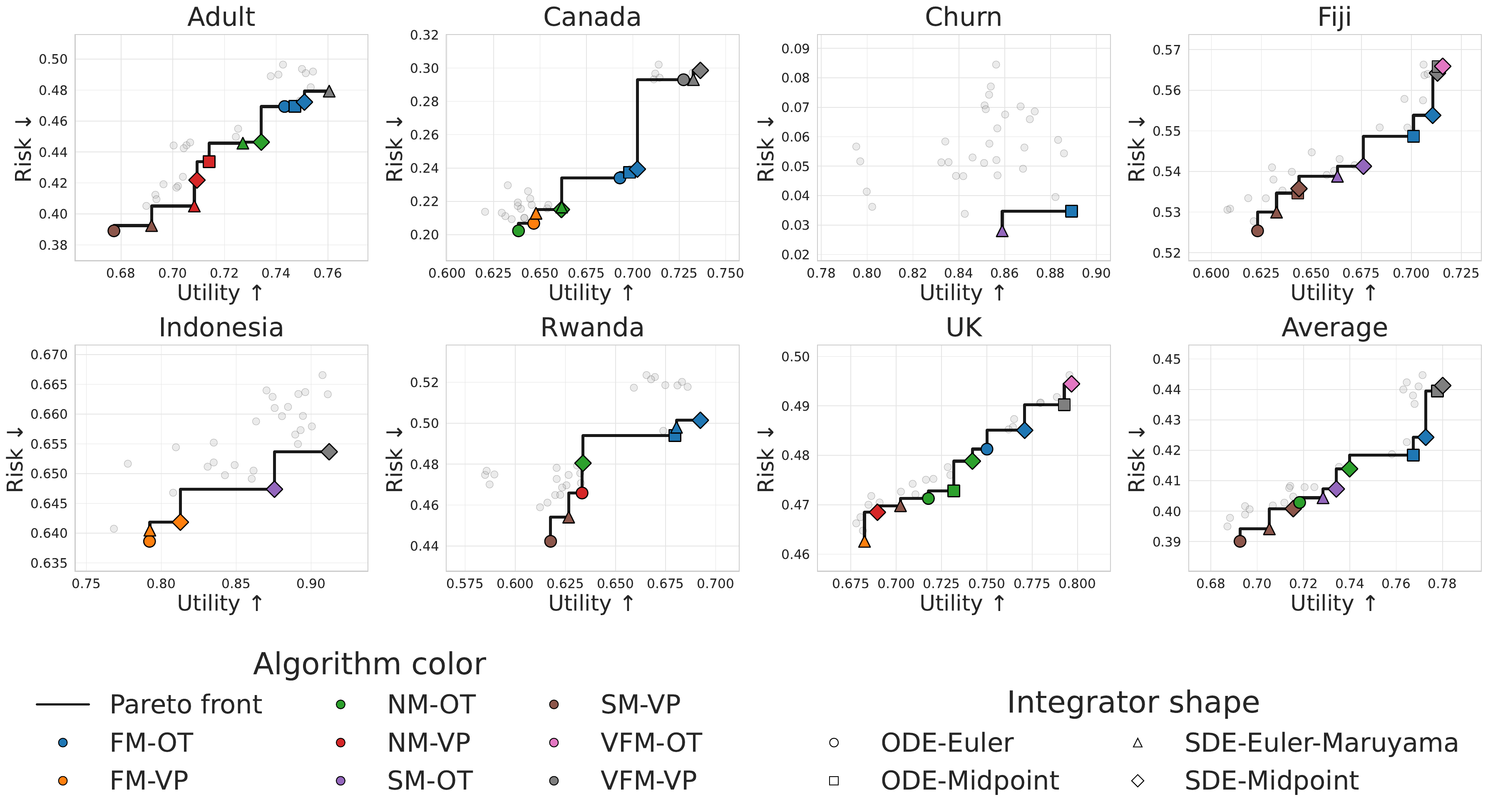}
    \caption{Pareto-style utility--disclosure risk summaries across datasets. Each panel shows the utility and disclosure risk of all target--path and integrator configurations for one dataset, with the final panel reporting the average across datasets. Utility is shown on the horizontal axis, where higher values are preferred, while disclosure risk is shown on the vertical axis, where lower values are preferred. Coloured markers denote target--path configurations and marker shapes denote integrators. Black lines highlight Pareto set configurations, for which no other configuration achieves both higher utility and lower disclosure risk.}
    \label{fig:placeholder}
\end{figure}

We also report rank-based summaries to reduce sensitivity to dataset-specific metric scales in Table~\ref{tab:app_q123_rank}. For each dataset, configurations are ranked separately for utility and disclosure risk, with higher utility ranked more favourably and lower disclosure risk ranked more favourably. The reported values are the mean and standard deviation of ranks across datasets; lower ranks therefore indicate more favourable positions.

The rank summaries reinforce the utility--risk trade-off observed in the main text. VFM-VP and FM-OT occupy the strongest utility ranks, whereas FM-VP and SM-VP tend to occupy the most conservative disclosure-risk ranks. This indicates that the lowest-risk configurations are not necessarily the same as the highest-utility configurations. We therefore interpret these results as supporting risk-constrained configuration selection rather than a single overall ranking.

\begin{table}[htbp]
\centering
\caption{Rank summary of utility and disclosure risk across datasets. For each dataset, configurations are ranked separately, with higher utility and lower disclosure risk receiving more favourable ranks. Values report the mean $\pm$ standard deviation of ranks across datasets. Lower ranks are better.}
\label{tab:app_q123_rank}
\begin{tabular}{w{c}{1.5cm} w{c}{2cm} w{c}{2cm} w{c}{2cm} w{c}{2cm} w{c}{2cm}} 
\toprule
Evaluation & Algorithm & ODE-Euler & ODE-Midpoint & SDE-EM & SDE-Midpoint\\
\midrule
\multirow{8}{*}{Utility}  & FM-OT & 9.71 ± 2.98 & 7.29 ± 3.35 & 7.71 ± 3.15 & 4.71 ± 2.75 \\
 & FM-VP & 28.29 ± 4.57 & 26.14 ± 2.73 & 28.00 ± 5.10 & 25.86 ± 4.41 \\
 & NM-OT & 21.57 ± 3.41 & 19.43 ± 4.28 & 14.57 ± 3.15 & 13.43 ± 1.90 \\
 & NM-VP & 26.43 ± 7.32 & 20.14 ± 4.85 & 22.43 ± 8.10 & 20.86 ± 6.01 \\
 & SM-OT & 23.29 ± 4.61 & 21.43 ± 4.58 & 17.00 ± 3.65 & 15.29 ± 3.25 \\
 & SM-VP & 28.29 ± 3.77 & 24.00 ± 4.32 & 25.29 ± 4.64 & 22.86 ± 3.02 \\
 & VFM-OT & 10.00 ± 3.83 & 7.43 ± 4.50 & 9.00 ± 5.42 & 7.00 ± 6.58 \\
 & VFM-VP & 7.00 ± 3.56 & 3.57 ± 1.72 & 7.29 ± 5.15 & 2.71 ± 1.50 \\
\midrule
 \multirow{8}{*}{Risk} & FM-OT & 19.43 ± 6.43 & 19.57 ± 7.41 & 22.43 ± 1.40 & 23.43 ± 3.78 \\
 & FM-VP & 4.57 ± 3.15 & 8.00 ± 4.32 & 7.14 ± 6.87 & 10.71 ± 4.96 \\
 & NM-OT & 10.14 ± 7.34 & 14.29 ± 6.50 & 19.29 ± 6.05 & 19.57 ± 6.13 \\
 & NM-VP & 9.29 ± 9.76 & 14.29 ± 5.28 & 10.86 ± 7.80 & 15.57 ± 7.21 \\
 & SM-OT & 10.71 ± 3.73 & 14.43 ± 4.93 & 11.71 ± 5.62 & 13.29 ± 5.85 \\
 & SM-VP & 5.14 ± 4.85 & 11.71 ± 6.52 & 6.14 ± 4.91 & 9.71 ± 3.99 \\
 & VFM-OT & 26.86 ± 3.58 & 28.14 ± 6.15 & 28.57 ± 1.81 & 30.14 ± 1.95 \\
 & VFM-VP & 23.29 ± 5.91 & 25.86 ± 4.67 & 26.57 ± 2.30 & 27.14 ± 5.58 \\
\bottomrule
\end{tabular}
\end{table}

\subsection{Q1--Q3 Additional Results}
\label{app:detail-res-q1q3}
Table~\ref{tab:app_q123_summary} reports three additional fidelity diagnostics: \(\alpha\)-precision, \(\beta\)-recall, and Wasserstein distance (WD). Overall, \(\alpha\)-precision is high for most configurations, making it less discriminative in this setting. In contrast, \(\beta\)-recall and WD show clearer separation between target--path configurations. VFM variants generally achieve stronger coverage and lower distributional discrepancy, while FM-OT remains competitive, especially under the default ODE-Euler setting. NM and SM variants tend to be slightly weaker on these auxiliary fidelity metrics, although SDE sampling can improve their coverage and WD. 

Solver and sampling effects are secondary but still visible. SDE-Midpoint frequently gives the best or second-best \(\beta\)-recall and WD values, particularly for NM and SM variants. Midpoint-based solvers also tend to reduce WD modestly compared with Euler at the same step count. These results support the main Q1--Q3 findings: the target--path choice largely determines the fidelity regime, while sampling dynamics and solver choice act as calibration tools rather than universally changing the ranking of configurations.

\begin{table}[htbp]
\centering
\caption{Additional aggregate comparison answering Q1--Q3 based on fidelity: learning target (rows), sampling dynamics (ODE vs.\ SDE), and numerical solver (Euler vs.\ Midpoint). We also include TabSyn results from~\cite{nasution2025flowmatchingtabulardata} as the baseline. Bold entries mark the best integrator and underline the second best for each dataset and metric. Values are averaged over 20 runs per dataset and then averaged across datasets.}
\label{tab:app_q123_summary}
\begin{tabular}{w{c}{1.5cm} w{c}{2cm} w{c}{2cm} w{c}{2cm} w{c}{2cm} w{c}{2cm}} 
\toprule
Evaluation & Algorithm & ODE-Euler & ODE-Midpoint & SDE-EM & SDE-Midpoint\\
\midrule
\multirow{8}{*}{$\alpha$-Pr. (\%,$\uparrow$)}  & FM-OT & 99.0585 & 98.9804 & \textbf{99.1770} & \underline{99.1056} \\
 & FM-VP & \underline{98.8180} & 98.7155 & \textbf{98.8280} & 98.6853 \\
 & NM-OT & 98.1938 & 98.3096 & \underline{98.7344} & \textbf{98.9175} \\
 & NM-VP & 96.8774 & \textbf{99.0568} & 98.0071 & \underline{99.0236} \\
 & SM-OT & 97.8428 & 98.1218 & \underline{98.4934} & \textbf{98.7896} \\
 & SM-VP & 97.4505 & \textbf{98.9718} & 98.3052 & \underline{98.8400} \\
 & VFM-OT & 98.9487 & \textbf{99.0400} & \underline{98.9846} & 98.9125 \\
 & VFM-VP & \textbf{99.3036} & 99.0316 & \underline{99.2842} & 99.0101 \\
\midrule
 \multirow{8}{*}{$\beta$-Re. (\%,$\uparrow$)} & FM-OT & \underline{62.2081} & 61.9859 & 62.2071 & \textbf{62.2390} \\
 & FM-VP & 60.1029 & \underline{60.4509} & 60.0728 & \textbf{60.5295} \\
 & NM-OT & 60.5542 & 60.6868 & \underline{60.9172} & \textbf{61.0350} \\
 & NM-VP & 59.6065 & \underline{60.5713} & 59.8345 & \textbf{60.6301} \\
 & SM-OT & 60.4276 & 60.4681 & \underline{60.8143} & \textbf{60.9561} \\
 & SM-VP & 59.8193 & \underline{60.6592} & 60.0315 & \textbf{60.6650} \\
 & VFM-OT & 63.8334 & \underline{63.9348} & 63.8429 & \textbf{63.9994} \\
 & VFM-VP & 63.5339 & \textbf{64.0051} & 63.5101 & \underline{63.9973} \\
\midrule
\multirow{8}{*}{WD ($\downarrow$)} & FM-OT & 1.2685 & 1.2661 & \underline{1.2653} & \textbf{1.2567} \\
 & FM-VP & 1.3350 & \textbf{1.3113} & 1.3344 & \underline{1.3116} \\
 & NM-OT & 1.3428 & 1.3323 & \underline{1.3221} & \textbf{1.3106} \\
 & NM-VP & 1.3631 & \underline{1.3119} & 1.3483 & \textbf{1.3082} \\
 & SM-OT & 1.3434 & 1.3332 & \underline{1.3255} & \textbf{1.3117} \\
 & SM-VP & 1.3618 & \underline{1.3080} & 1.3463 & \textbf{1.3060} \\
 & VFM-OT & 1.2216 & \underline{1.2133} & 1.2195 & \textbf{1.2085} \\
 & VFM-VP & 1.2345 & \underline{1.2122} & 1.2342 & \textbf{1.2118} \\
\bottomrule
\end{tabular}
\end{table}

\subsection{Integration Endpoint ($t_{\mathrm{int}}$)}
\label{sec:res-t-ode-detailed}

While Section~\ref{sec:res-t-ode} reports averages, Figure~\ref{fig:t-ode-full} details the dataset-specific utility-risk trajectories as $t_{\mathrm{int}}$ increases from 0.6 to 1.0. The key observation is the pronounced heterogeneity across domains: a uniform integration strategy is rarely optimal, as the marginal utility gains and risk penalties vary drastically by dataset.

We identify three distinct behavioral patterns. First, some data such as \textbf{Churn} and \textbf{Canada} exhibit smoother trajectories where risk remains comparatively low ($R < 0.3$) even at full integration. In Churn specifically, risk tends to be negligible across the entire trajectory, making $t_{\mathrm{int}}$ purely a utility-cost decision. Here, simpler baselines (e.g., NM) remain competitive with VFM variants.

Second, domains like \textbf{UK} and \textbf{Adult} display a classic "knee" shape. Utility improves steadily, but risk acceleration becomes noticeable in the late-integration regime ($t_{\mathrm{int}} > 0.9$), suggesting that a slightly truncated endpoint (e.g., $t_{\mathrm{int}} \approx 0.95$) offers a pragmatic balance, sacrificing minimal utility for meaningful risk reduction.

Third, high-risk domains such as \textbf{Indonesia}, \textbf{Fiji}, and \textbf{Rwanda} show a steep increase in risk. Their trajectories curve sharply upward in the final stages, indicating that the last increments in utility come at a disproportionate disclosure risk. In these cases, blindly defaulting to full integration is detrimental. Early stopping or risk-constrained integration endpoints are recommended to maintain an acceptable risk.

Across all datasets, the structural difference between OT and VP trajectories is evident. OT-based methods (solid lines) consistently reach higher utility tiers earlier in the integration process compared to VP (dashed lines). This characteristic makes OT variants recommended for truncation, offering a safer fallback when computational or risk constraints require early stopping. In contrast, VP variants generally require nearly complete integration to be competitive.

\begin{figure}[ht]
    \centering
    \includegraphics[width=\linewidth]{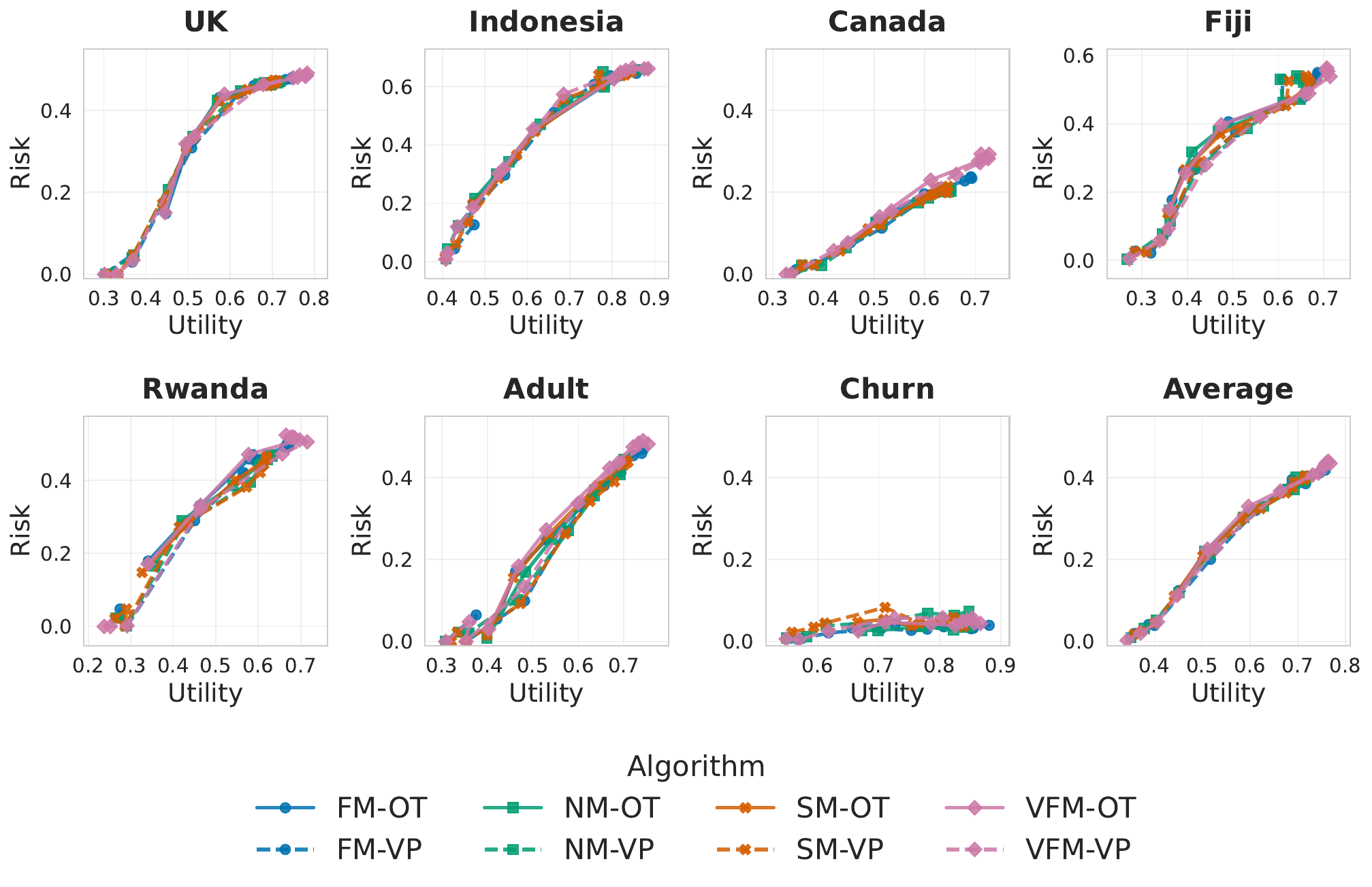}
    \caption{Dataset-specific risk-utility trajectories for varying integration times $t_{\mathrm{int}} \in [0.6, 1.0]$. Benign datasets (Churn, Canada) show linear progression with manageable risk, while challenging domains (Fiji, Rwanda, Indonesia) exhibit a sharp "risk wall" at high integration times. OT-based methods (solid lines) typically achieve higher utility earlier in the trajectory than VP-based methods (dashed lines), demonstrating greater resilience to truncation.}
    \label{fig:t-ode-full}
\end{figure}

\subsection{Integration Steps}
\label{sec:nfe-detailed}

Figure~\ref{fig:nfe-all} shows dataset-specific utility--risk trajectories as the sampling budget increases from 4 to 1024 steps. We focus on within-dataset trends, since the absolute utility and risk ranges differ materially across domains. Overall, the figure highlights that the marginal utility gained per additional compute, and the accompanying risk growth, are strongly dataset-dependent.

In \textbf{Churn}, risk remains low across the entire budget range while utility is consistently high, so increasing steps primarily trades compute for incremental utility improvements. \textbf{Canada} also exhibits relatively gradual risk growth as utility increases. In contrast, \textbf{Fiji} and \textbf{Rwanda} show steeper frontiers where risk increases rapidly as one pushes towards higher utility, and additional compute beyond mid-range budgets yields diminishing utility gains while risk continues to accumulate. These patterns reinforce that selecting steps can also be risk-aware rather than solely compute-driven.

\begin{figure}[h]
    \centering
    \includegraphics[width=0.95\linewidth]{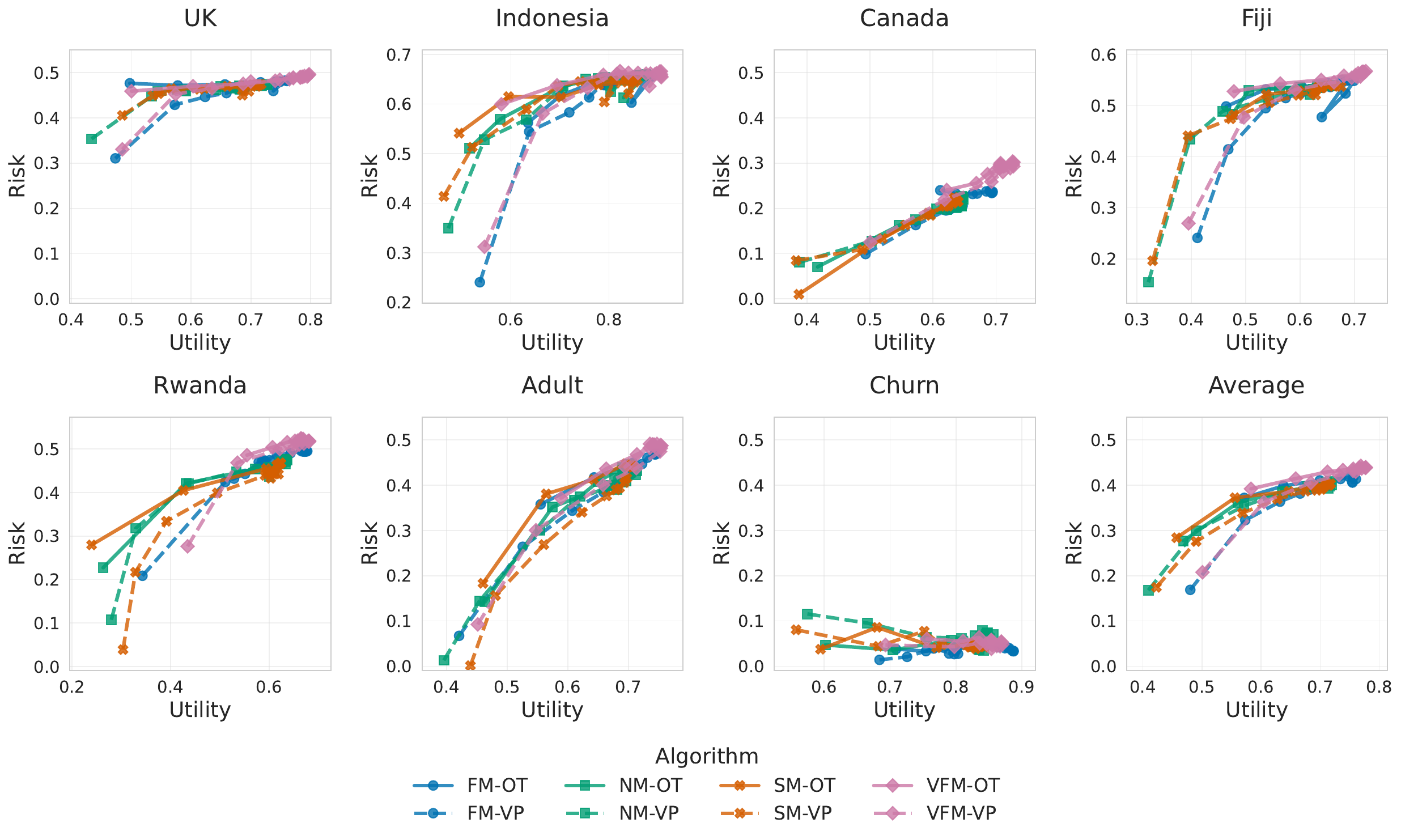}
    \caption{Dataset-specific utility--risk trajectories as a function of sampling budget (steps $=2^n$ with $n\in[2,10]$). Each line connects budgets from small to large integration steps, showing how increased compute moves a method along the utility--risk frontier. Axis limits are set per dataset to preserve readability of within-dataset trends. For a direct comparison, we refer to the aggregated curves reported in Figure~\ref{fig:res_nfe}.}
    \label{fig:nfe-all}
\end{figure}

\clearpage
\newpage

\subsection{Detailed Evaluation Results - External Comparison}

We report detailed performance of comparesion of our algorithm with external baselines in Tables~\ref{tab:external-detail-u}--~\ref{tab:external-detail-wd}. All entries are mean $\pm$ standard deviation across 20 seeds. Some values may coincide at the reported precision due to rounding.

\begin{table}[h]
    \centering
    \caption{Average utility of selected flow models compared to external baselines across datasets. Bold entries mark the best algorithm and underline the second best for each dataset and metric.}
\begin{tabular}{lccccccc}
\toprule
Algorithm & Adult & Canada & Churn & Fiji & Indonesia & Rwanda & UK \\
\midrule
Bayesian Network & 0.58±0.01 & 0.56±0.00 & \underline{0.87±0.01} & 0.50±0.01 & 0.74±0.01 & 0.56±0.01 & 0.59±0.01 \\
TVAE & 0.48±0.01 & 0.5±0.01 & 0.72±0.01 & 0.46±0.01 & 0.56±0.01 & 0.4±0.02 & 0.48±0.01 \\
Robust TVAE & 0.40±0.01 & 0.49±0.01 & 0.67±0.02 & 0.41±0.01 & 0.43±0.00 & 0.36±0.01 & 0.41±0.00 \\
CTGAN & 0.47±0.01 & 0.57±0.01 & 0.45±0.01 & 0.48±0.01 & 0.6±0.01 & 0.45±0.01 & 0.59±0.02 \\
Normalising Flow & 0.48±0.01 & 0.53±0.01 & 0.71±0.01 & 0.44±0.01 & 0.53±0.00 & 0.43±0.01 & 0.49±0.00 \\
TabDDPM~\cite{nasution2025flowmatchingtabulardata} & 0.67±0.02 & 0.28±0.00 & 0.82±0.01 & 0.13±0.00 & 0.8±0.01 & 0.41±0.01 & \textbf{0.78±0.02} \\
TabSyn~\cite{nasution2025flowmatchingtabulardata} & \textbf{0.75±0.01} & \textbf{0.74±0.02} & 0.87±0.01 & \textbf{0.71±0.02} & 0.87±0.03 & \underline{0.67±0.02} & \underline{0.76±0.02} \\
FM-OT & 0.74±0.01 & 0.69±0.02 & \textbf{0.88±0.02} & 0.68±0.02 & \textbf{0.88±0.03} & 0.67±0.02 & 0.75±0.01 \\
VFM-VP & \underline{0.75±0.01} & \underline{0.73±0.01} & 0.87±0.02 & \underline{0.71±0.03} & \underline{0.88±0.03} & \textbf{0.68±0.03} & 0.76±0.02 \\
\bottomrule
\end{tabular}
    \label{tab:external-detail-u}
\end{table}

\begin{table}[h]
    \centering
    \caption{Average risk of selected flow models compared to external baselines across datasets. Bold entries mark the best algorithm and underline the second best for each dataset and metric.}
\begin{tabular}{lccccccc}
\toprule
Algorithm & Adult & Canada & Churn & Fiji & Indonesia & Rwanda & UK \\
\midrule
Bayesian Network & 0.27±0.03 & 0.2±0.02 & 0.35±0.08 & 0.51±0.01 & 0.56±0.01 & 0.42±0.01 & \underline{0.26±0.01} \\
TVAE & 0.23±0.02 & 0.26±0.02 & \underline{0.01±0.04} & 0.48±0.01 & 0.45±0.05 & 0.51±0.01 & 0.44±0.01 \\
Robust TVAE & \textbf{0.00±0.00} & \underline{0.04±0.04} & \textbf{0.00±0.00} & \underline{0.18±0.10} & \textbf{0.00±0.00} & \underline{0.18±0.14} & \textbf{0.00±0.00} \\
CTGAN & \underline{0.00±0.00} & 0.27±0.02 & 0.19±0.02 & 0.26±0.01 & \underline{0.24±0.05} & 0.41±0.04 & 0.37±0.01 \\
Normalising Flow & 0.00±0.00 & 0.19±0.03 & 0.04±0.06 & 0.44±0.01 & 0.34±0.03 & 0.36±0.02 & 0.36±0.01 \\
TabDDPM~\cite{nasution2025flowmatchingtabulardata} & 0.45±0.02 & \textbf{0.03±0.07} & 0.07±0.05 & \textbf{0.05±0.09} & 0.65±0.01 & \textbf{0.07±0.16} & 0.50±0.01 \\
TabSyn~\cite{nasution2025flowmatchingtabulardata} & 0.48±0.01 & 0.27±0.01 & 0.07±0.10 & 0.56±0.01 & 0.66±0.01 & 0.50±0.01 & 0.49±0.01 \\
FM-OT & 0.47±0.02 & 0.23±0.02 & 0.04±0.05 & 0.55±0.01 & 0.66±0.01 & 0.50±0.02 & 0.48±0.01 \\
VFM-VP & 0.48±0.02 & 0.29±0.01 & 0.05±0.06 & 0.56±0.01 & 0.66±0.01 & 0.52±0.01 & 0.49±0.01 \\
\bottomrule
\end{tabular}
    \label{tab:external-detail-r}
\end{table}

\begin{table}[h]
    \centering
    \caption{Average shape of selected flow models compared to external baselines across datasets. Bold entries mark the best algorithm and underline the second best for each dataset and metric.}
\begin{tabular}{lccccccc}
\toprule
Algorithm & Adult & Canada & Churn & Fiji & Indonesia & Rwanda & UK \\
\midrule
Bayesian Network & 17.54±0.03 & 11.95±0.05 & 3.12±0.08 & 12.48±0.05 & 0.54±0.04 & 11.56±0.07 & 13.02±0.04 \\
TVAE & 20.83±0.05 & 11.8±0.04 & 5.07±0.12 & 14.57±0.03 & 2.58±0.04 & 9.24±0.04 & 16.32±0.04 \\
Robust TVAE & 33.01±0.05 & 15.62±0.04 & 10.10±0.13 & 20.43±0.05 & 15.90±0.03 & 14.76±0.04 & 19.14±0.03 \\
CTGAN & 27.91±0.05 & 10.2±0.05 & 22.57±0.13 & 14.20±0.05 & 2.02±0.03 & 12.30±0.06 & 15.09±0.04 \\
Normalising Flow & 22.67±0.07 & 13.76±0.07 & 6.83±0.17 & 14.00±0.04 & 4.77±0.03 & 15.62±0.09 & 18.68±0.06 \\
TabDDPM~\cite{nasution2025flowmatchingtabulardata} & 1.56±0.03 & 40.58±0.12 & 1.16±0.08 & 50.74±0.12 & 0.34±0.03 & 24.28±0.12 & 0.95±0.04 \\
TabSyn~\cite{nasution2025flowmatchingtabulardata}  & 0.99±0.04 & 3.30±0.07  & 1.04±0.13 & 1.04±0.01  & 0.46±0.03 & 0.94±0.05  & 0.93±0.04 \\
FM-OT & \textbf{0.89±0.05} & 3.74±0.04 & \textbf{1.03±0.31} & \underline{1.06±0.04} & \textbf{0.24±0.03} & \underline{0.83±0.05} & \underline{0.85±0.03} \\
VFM-VP & 1.04±0.05 & \underline{3.64±0.05} & 1.23±0.08 & 1.17±0.06 & \underline{0.24±0.03} & \textbf{0.78±0.07} & \textbf{0.74±0.03} \\
\bottomrule
\end{tabular}
    \label{tab:external-detail-sh}
\end{table}

\begin{table}[h]
    \centering
    \caption{Average trend of selected flow models compared to external baselines across datasets. Bold entries mark the best algorithm and underline the second best for each dataset and metric.}
\begin{tabular}{lccccccc}
\toprule
Algorithm & Adult & Canada & Churn & Fiji & Indonesia & Rwanda & UK \\
\midrule
Bayesian Network & 18.83±0.10 & 29.48±0.18 & 3.84±0.45 & 21.44±0.07 & 3.05±0.04 & 19.38±0.12 & 23.54±0.06 \\
TVAE & 25.78±0.78 & 28.81±0.11 & 12.08±0.57 & 34.46±0.04 & 4.36±0.17 & 14.36±0.04 & 29.68±0.05 \\
Robust TVAE & 51.39±4.38 & 48.44±0.06 & 25.10±0.11 & 65.35±0.02 & 27.86±0.03 & 38.07±0.06 & 38.83±0.03 \\
CTGAN & 27.88±0.22 & 18.89±1.51 & 34.23±0.26 & 40.8±0.05 & 4.46±0.05 & 21.29±0.06 & 26.1±0.06 \\
Normalising Flow & 25.81±0.14 & 23.55±0.13 & 8.52±0.16 & 22.91±0.05 & 9.07±0.03 & 25.42±0.12 & 30.61±0.06 \\
TabDDPM~\cite{nasution2025flowmatchingtabulardata} & 3.64±0.06 & 61.30±0.13 & \textbf{2.38±0.15} & 78.47±0.05 & 0.61±0.05 & 45.19±0.21 & \underline{1.79±0.05} \\
TabSyn~\cite{nasution2025flowmatchingtabulardata}  & \textbf{2.11±0.07} & \underline{8.05±0.52} & \underline{2.41±0.25} & \textbf{1.96±0.01} & 0.78±0.04 & 1.37±0.09 & 1.91±0.04 \\
FM-OT & \underline{2.13±0.14} & 8.25±0.71 & 2.46±0.52 & 2.11±0.05 & \textbf{0.46±0.05} & \underline{1.29±0.08} & 1.86±0.05 \\
VFM-VP & 2.25±0.15 & \textbf{6.75±0.13} & 2.88±0.17 & \underline{2.06±0.06} & \underline{0.47±0.03} & \textbf{1.24±0.07} & \textbf{1.75±0.06} \\
\bottomrule
\end{tabular}
    \label{tab:external-detail-tr}
\end{table}

\begin{table}[h]
    \centering
    \caption{Average $\alpha$-precision of selected flow models compared to external baselines across datasets. Bold entries mark the best algorithm and underline the second best for each dataset and metric.}
\begin{tabular}{lccccccc}
\toprule
Algorithm & Adult & Canada & Churn & Fiji & Indonesia & Rwanda & UK \\
\midrule
Bayesian Network & 61.43±0.29 & 65.84±0.31 & 98.84±0.41 & 66.38±0.17 & 96.95±0.10 & 57.34±0.3 & 60.01±0.17 \\
TVAE & 85.03±0.23 & 72.06±0.18 & 98.79±0.28 & 51.39±0.09 & 93.46±0.13 & 76.06±0.19 & 35.73±0.11 \\
Robust TVAE & 26.42±0.08 & 40.92±0.08 & 95.05±0.54 & 35.32±0.06 & 63.49±0.09 & 46.33±0.13 & 60.04±0.08 \\
CTGAN & 50.80±0.16 & 73.53±0.22 & 89.01±0.39 & 69.84±0.14 & 98.60±0.10 & 61.03±0.20 & 53.04±0.21 \\
Normalising Flow & 45.11±0.13 & 60.78±0.28 & 97.84±0.35 & 66.75±0.17 & 91.02±0.13 & 44.47±0.23 & 49.29±0.13 \\
TabDDPM~\cite{nasution2025flowmatchingtabulardata} & 95.42±0.17 & 30.20±0.17 & 96.87±0.37 & 50.99±0.03 & 98.96±0.12 & 49.93±0.37 & 96.95±0.15 \\
TabSyn~\cite{nasution2025flowmatchingtabulardata}  & 98.50±0.22 & 98.14±0.30 & 98.56±0.37 & 98.46±0.15 & 98.50±0.11 & 97.97±0.32 & 98.67±0.21 \\
FM-OT & \underline{99.21±0.14} & \textbf{98.66±0.20} & \textbf{99.10±0.35} & \underline{99.15±0.16} & \underline{99.43±0.11} & \underline{98.92±0.15} & \underline{98.93±0.16} \\
VFM-VP & \textbf{99.41±0.11} & \underline{98.65±0.22} & \underline{99.07±0.40} & \textbf{99.54±0.11} & \textbf{99.81±0.04} & \textbf{99.04±0.16} & \textbf{99.61±0.14} \\
\bottomrule
\end{tabular}
    \label{tab:external-detail-al}
\end{table}

\begin{table}[h]
    \centering
    \caption{Average $\beta$-recall of selected flow models compared to external baselines across datasets. Bold entries mark the best algorithm and underline the second best for each dataset and metric.}
\begin{tabular}{lccccccc}
\toprule
Algorithm & Adult & Canada & Churn & Fiji & Indonesia & Rwanda & UK \\
\midrule
Bayesian Network & 32.37±0.18 & 9.57±0.20 & \textbf{63.26±0.56} & 29.65±0.24 & 93.0±0.06 & 59.25±0.37 & 35.94±0.14 \\
TVAE & 13.15±0.13 & 13.89±0.29 & 47.5±0.64 & 1.95±0.05 & 93.47±0.07 & 56.02±0.89 & 6.96±0.25 \\
Robust TVAE & 0.65±0.02 & 1.63±0.05 & 45.16±0.39 & 0.80±0.02 & 17.65±0.02 & 1.85±0.13 & 9.97±0.07 \\
CTGAN & 2.47±0.07 & 15.37±0.28 & 33.47±0.40 & 8.13±0.17 & 89.95±0.14 & 11.84±1.10 & 21.51±0.33 \\
Normalising Flow & 3.84±0.10 & 11.22±0.24 & 45.31±0.30 & 9.95±0.20 & 91.02±0.08 & 14.10±1.34 & 6.03±0.15 \\
TabDDPM~\cite{nasution2025flowmatchingtabulardata} & 45.35±0.19 & 0.45±0.04 & \underline{54.22±0.67} & 0.02±0.01 & 96.35±0.06 & 0.87±0.28 & \textbf{68.84±0.13} \\
TabSyn~\cite{nasution2025flowmatchingtabulardata}  & \underline{48.78±0.28} & \underline{33.39±0.19} & 50.93±0.51 & \underline{56.79±0.17} & 96.47±0.04 & 84.52±0.20 & 67.58±0.15 \\
FM-OT & 47.38±0.33 & 32.24±0.26 & 50.75±0.56 & 55.44±0.20 & \underline{96.43±0.05} & \underline{85.51±0.22} & 67.69±0.15 \\
VFM-VP & \textbf{48.82±0.31} & \textbf{35.76±0.30} & 51.75±0.47 & \textbf{57.44±0.17} & \textbf{96.56±0.05} & \textbf{85.86±0.25} & \underline{68.55±0.19} \\
\bottomrule
\end{tabular}
    \label{tab:external-detail-be}
\end{table}

\begin{table}[h]
    \centering
    \caption{Average Wasserstein Distance of selected flow models compared to external baselines across datasets. Bold entries mark the best algorithm and underline the second best for each dataset and metric.}
\begin{tabular}{lccccccc}
\toprule
Algorithm & Adult & Canada & Churn & Fiji & Indonesia & Rwanda & UK \\
\midrule
Bayesian Network & 1.98±0.01 & 4.91±0.01 & \textbf{0.15±0.00} & 4.45±0.01 & 0.45±0.00 & 2.63±0.01 & 3.43±0.00 \\
TVAE & 1.80±0.00 & 3.72±0.01 & 0.22±0.00 & 4.71±0.00 & 0.26±0.00 & 1.81±0.00 & 3.44±0.00 \\
Robust TVAE & 2.86±0.00 & 4.87±0.00 & 0.34±0.00 & 5.67±0.00 & 1.41±0.00 & 3.11±0.00 & 3.64±0.00 \\
CTGAN & 2.49±0.00 & 3.72±0.01 & 1.09±0.01 & 4.13±0.00 & 0.28±0.00 & 2.57±0.00 & 3.15±0.00 \\
Normalising Flow & 2.49±0.00 & 4.12±0.01 & 0.27±0.00 & 4.42±0.00 & 0.53±0.00 & 3.15±0.01 & 3.58±0.00 \\
TabDDPM~\cite{nasution2025flowmatchingtabulardata} & 0.84±0.00 & 8.95±0.02 & \underline{0.15±0.00} & 9.28±0.01 & \textbf{0.17±0.00} & 6.00±0.02 & \textbf{1.71±0.00} \\
TabSyn~\cite{nasution2025flowmatchingtabulardata} & \textbf{0.80±0.00} & \underline{2.56±0.01} & 0.16±0.00 & \underline{2.56±0.00} & \underline{0.18±0.00} & \underline{0.73±0.00} & \underline{1.77±0.00} \\
FM-OT & 0.82±0.01 & 2.60±0.01 & 0.16±0.00 & 2.60±0.00 & 0.18±0.00 & 0.74±0.00 & 1.78±0.00 \\
VFM-VP & \underline{0.80±0.01} & \textbf{2.49±0.01} & 0.16±0.00 & \textbf{2.55±0.00} & 0.18±0.00 & \textbf{0.70±0.00} & 1.77±0.00 \\
\bottomrule
\end{tabular}
    \label{tab:external-detail-wd}
\end{table}

\clearpage
\newpage

\subsection{Detailed Evaluation Results based on Data - ODE}
We report detailed performance under ODE sampling with 100 integration steps across seven datasets. Tables~\ref{tab:res-baseline-u}--\ref{tab:res-baseline-wd} use Euler integration, and Tables~\ref{tab:res-baseline-u-mp}--\ref{tab:res-baseline-wd-mp} use the Midpoint solver. All entries are mean $\pm$ standard deviation across 20 seeds. Some values may coincide at the reported precision due to rounding.

\begin{table}[h]
    \centering
    \caption{Average utility of flow models across datasets. Bold entries mark the best algorithm and underline the second best for each dataset and metric. Note that this table's results uses Euler ODE integration with 100 steps.}
    \begin{tabular}{lcccccccc}
    \toprule
    Dataset & FM-OT & FM-VP & SM-OT & SM-VP & NM-OT & NM-VP & VFM-OT & VFM-VP \\
    \midrule
    Adult & \underline{0.74±0.01} & 0.69±0.02 & 0.70±0.02 & 0.68±0.01 & 0.70±0.02 & 0.69±0.01 & 0.74±0.02 & \textbf{0.75±0.01} \\
    Canada & 0.69±0.02 & 0.65±0.01 & 0.63±0.01 & 0.62±0.01 & 0.64±0.01 & 0.63±0.01 & \underline{0.71±0.02} & \textbf{0.73±0.01} \\
    Churn & \textbf{0.88±0.02} & 0.80±0.02 & 0.83±0.02 & 0.84±0.02 & 0.84±0.02 & 0.85±0.02 & 0.85±0.02 & \underline{0.87±0.02} \\
    Fiji & 0.68±0.02 & 0.61±0.02 & 0.66±0.02 & 0.62±0.02 & 0.64±0.02 & 0.61±0.02 & \textbf{0.71±0.03} & \underline{0.71±0.03} \\
    Indonesia & \textbf{0.88±0.03} & 0.79±0.02 & 0.85±0.03 & 0.77±0.01 & 0.86±0.03 & 0.78±0.01 & 0.88±0.03 & \underline{0.88±0.03} \\
    Rwanda & \underline{0.67±0.02} & 0.59±0.02 & 0.62±0.01 & 0.62±0.02 & 0.63±0.02 & 0.63±0.02 & 0.67±0.03 & \textbf{0.68±0.03} \\
    UK & 0.75±0.01 & 0.68±0.01 & 0.71±0.02 & 0.70±0.01 & 0.72±0.01 & 0.68±0.01 & \textbf{0.78±0.02} & \underline{0.76±0.02} \\
\bottomrule
\end{tabular}
    \label{tab:res-baseline-u}
\end{table}

\begin{table}[ht]
    \centering
    \caption{Average risk of flow models across datasets. Bold entries mark the best algorithm and underline the second best for each dataset and metric. Note that this table's results uses Euler ODE integration with 100 steps.}
    \begin{tabular}{lcccccccc}
    \toprule
    Dataset & FM-OT & FM-VP & SM-OT & SM-VP & NM-OT & NM-VP & VFM-OT & VFM-VP \\
    \midrule
    Adult & 0.47±0.02 & \underline{0.41±0.01} & 0.44±0.02 & \textbf{0.39±0.02} & 0.44±0.02 & 0.41±0.02 & 0.49±0.02 & 0.48±0.02 \\
    Canada & 0.23±0.02 & \underline{0.21±0.02} & 0.21±0.02 & 0.21±0.01 & \textbf{0.20±0.01} & 0.21±0.02 & 0.29±0.02 & 0.29±0.01 \\
    Churn & \underline{0.04±0.05} & 0.04±0.06 & 0.05±0.07 & 0.05±0.06 & \textbf{0.03±0.05} & 0.07±0.06 & 0.06±0.06 & 0.05±0.06 \\
    Fiji & 0.55±0.01 & \textbf{0.53±0.01} & 0.54±0.01 & \underline{0.53±0.01} & 0.54±0.01 & 0.53±0.01 & 0.56±0.00 & 0.56±0.01 \\
    Indonesia & 0.66±0.01 & \textbf{0.64±0.01} & 0.65±0.01 & \underline{0.64±0.02} & 0.66±0.01 & 0.65±0.01 & 0.66±0.01 & 0.66±0.01 \\
    Rwanda & 0.50±0.02 & 0.47±0.01 & \underline{0.46±0.03} & \textbf{0.44±0.02} & 0.47±0.02 & 0.47±0.02 & 0.52±0.01 & 0.52±0.01 \\
    UK & 0.48±0.01 & \textbf{0.46±0.01} & \underline{0.47±0.01} & 0.47±0.01 & 0.47±0.01 & 0.47±0.01 & 0.49±0.01 & 0.49±0.01 \\
    \bottomrule
    \end{tabular}
    \label{tab:res-baseline-r}
\end{table}

\begin{table}[ht]
    \centering
    \caption{Average error in \textbf{shape} (\%) of flow models across datasets. Bold entries mark the best algorithm and underline the second best for each dataset and metric. Note that this table's results uses Euler ODE integration with 100 steps.}
    \begin{tabular}{lcccccccc}
    \toprule
    Dataset & FM-OT & FM-VP & SM-OT & SM-VP & NM-OT & NM-VP & VFM-OT & VFM-VP \\
    \midrule
    Adult & \textbf{0.89±0.05} & 1.33±0.05 & \underline{1.03±0.06} & 1.44±0.05 & 1.27±0.05 & 1.71±0.06 & 1.18±0.05 & 1.04±0.05 \\
    Canada & \underline{3.74±0.04} & 4.53±0.05 & 4.74±0.06 & 4.87±0.06 & 4.66±0.07 & 4.74±0.10 & 3.83±0.05 & \textbf{3.64±0.05} \\
    Churn & \textbf{1.03±0.31} & 2.51±0.13 & 1.98±0.29 & 1.82±0.16 & 1.84±0.24 & 1.66±0.22 & 1.39±0.11 & \underline{1.23±0.08} \\
    Fiji & \textbf{1.06±0.04} & 1.93±0.05 & 1.50±0.06 & 2.01±0.12 & 1.31±0.06 & 2.25±0.06 & 1.25±0.04 & \underline{1.17±0.06} \\
    Indonesia & \textbf{0.24±0.03} & 0.34±0.03 & 0.37±0.04 & 0.44±0.07 & 0.42±0.06 & 0.54±0.05 & 0.24±0.03 & \underline{0.24±0.03} \\
    Rwanda & \underline{0.83±0.05} & 1.67±0.04 & 1.18±0.06 & 1.56±0.05 & 1.26±0.04 & 1.66±0.05 & 0.97±0.04 & \textbf{0.78±0.07} \\
    UK & 0.85±0.03 & 1.44±0.04 & 1.00±0.04 & 0.98±0.05 & 0.85±0.04 & 1.06±0.04 & \underline{0.82±0.04} & \textbf{0.74±0.03} \\
    \bottomrule
    \end{tabular}
    \label{tab:res-baseline-sh}
\end{table}

\begin{table}[ht]
    \centering
    \caption{Average error in \textbf{trend} (\%) of flow models across datasets. Bold entries mark the best algorithm and underline the second best for each dataset and metric. Note that this table's results uses Euler ODE integration with 100 steps.}
    \begin{tabular}{lcccccccc}
    \toprule
    Dataset & FM-OT & FM-VP & SM-OT & SM-VP & NM-OT & NM-VP & VFM-OT & VFM-VP \\
    \midrule
    Adult & \textbf{2.13±0.14} & 2.93±0.06 & 2.28±0.09 & 2.99±0.05 & 2.43±0.08 & 3.09±0.07 & 2.51±0.08 & \underline{2.25±0.15} \\
    Canada & 8.25±0.71 & 9.55±0.53 & 9.27±0.59 & 8.75±0.20 & 9.20±0.62 & 8.42±0.14 & \underline{7.15±1.74} & \textbf{6.75±0.13} \\
    Churn & \textbf{2.46±0.52} & 4.67±0.21 & 3.76±0.45 & 3.41±0.21 & 3.39±0.33 & 3.06±0.29 & 3.12±0.18 & \underline{2.88±0.17} \\
    Fiji & 2.11±0.05 & 3.48±0.05 & 2.80±0.06 & 3.58±0.17 & 2.66±0.06 & 3.90±0.06 & \textbf{2.02±0.05} & \underline{2.06±0.06} \\
    Indonesia & \underline{0.46±0.05} & 0.74±0.03 & 0.70±0.05 & 0.83±0.08 & 0.72±0.07 & 0.90±0.06 & \textbf{0.44±0.04} & 0.47±0.03 \\
    Rwanda & \underline{1.29±0.08} & 2.38±0.07 & 2.33±0.08 & 2.99±0.22 & 2.47±0.07 & 2.84±0.09 & 1.44±0.07 & \textbf{1.24±0.07} \\
    UK & 1.86±0.05 & 3.19±0.09 & 2.29±0.07 & 2.42±0.07 & 2.11±0.07 & 2.67±0.05 & \textbf{1.72±0.05} & \underline{1.75±0.06} \\
    \bottomrule
    \end{tabular}
    \label{tab:res-baseline-t}
\end{table}

\begin{table}[ht]
    \centering
    \caption{Average \textbf{$\alpha$-precision} (\%) of flow models across datasets. Bold entries mark the best algorithm and underline the second best for each dataset and metric. Note that this table's results uses Euler ODE integration with 100 steps.}
    \begin{tabular}{lcccccccc}
    \toprule
    Dataset & FM-OT & FM-VP & SM-OT & SM-VP & NM-OT & NM-VP & VFM-OT & VFM-VP \\
    \midrule
    Adult & \underline{99.21±0.14} & 99.14±0.08 & 98.70±0.21 & 98.02±0.28 & 97.91±0.20 & 96.19±0.24 & 99.15±0.15 & \textbf{99.41±0.11} \\
    Canada & \textbf{98.66±0.20} & 98.30±0.13 & 93.49±0.33 & 94.38±0.30 & 95.05±0.31 & 94.95±0.32 & 98.38±0.28 & \underline{98.65±0.22} \\
    Churn & \textbf{99.10±0.35} & 97.76±0.50 & 96.82±0.67 & 98.29±0.55 & 97.90±0.78 & 97.23±0.98 & 98.89±0.48 & \underline{99.07±0.40} \\
    Fiji & 99.15±0.16 & 99.28±0.09 & 99.23±0.14 & 97.65±0.17 & \textbf{99.65±0.09} & 96.48±0.17 & 99.00±0.15 & \underline{99.54±0.11} \\
    Indonesia & 99.43±0.11 & \underline{99.60±0.07} & 99.42±0.10 & 99.56±0.08 & 99.58±0.09 & 98.93±0.12 & 99.40±0.12 & \textbf{99.81±0.04} \\
    Rwanda & \underline{98.92±0.15} & 98.11±0.19 & 97.69±0.3 & 94.85±0.27 & 97.60±0.31 & 95.29±0.74 & 98.58±0.25 & \textbf{99.04±0.16} \\
    UK & 98.93±0.16 & 99.54±0.06 & 99.55±0.11 & 99.40±0.15 & \textbf{99.67±0.07} & 99.07±0.17 & 99.24±0.15 & \underline{99.61±0.14} \\
    \bottomrule
    \end{tabular}
    \label{tab:res-baseline-ap}
\end{table}

\begin{table}[ht]
    \centering
    \caption{Average \textbf{$\beta$-recall} (\%) of flow models across datasets. Bold entries mark the best algorithm and underline the second best for each dataset and metric. Note that this table's results uses Euler ODE integration with 100 steps.}
    \begin{tabular}{lcccccccc}
    \toprule
    Dataset & FM-OT & FM-VP & SM-OT & SM-VP & NM-OT & NM-VP & VFM-OT & VFM-VP \\
    \midrule
    Adult & 47.38±0.33 & 45.55±0.24 & 46.03±0.21 & 44.71±0.19 & 45.84±0.19 & 44.28±0.25 & \textbf{49.10±0.28} & \underline{48.82±0.31} \\
    Canada & 32.24±0.26 & 29.22±0.24 & 28.67±0.26 & 28.12±0.26 & 28.66±0.28 & 28.30±0.22 & \textbf{36.09±0.34} & \underline{35.76±0.30} \\
    Churn & 50.75±0.56 & 50.94±0.56 & 48.68±0.51 & 49.69±0.57 & 49.31±0.81 & 49.77±0.88 & \underline{51.73±0.57} & \textbf{51.75±0.47} \\
    Fiji & 55.44±0.20 & 50.73±0.17 & 53.25±0.22 & 50.29±0.22 & 53.05±0.20 & 50.07±0.20 & \textbf{58.46±0.23} & \underline{57.44±0.17} \\
    Indonesia & 96.43±0.05 & 96.35±0.07 & 96.44±0.06 & 96.46±0.05 & 96.47±0.07 & 96.45±0.07 & \underline{96.50±0.06} & \textbf{96.56±0.05} \\
    Rwanda & 85.51±0.22 & 83.72±0.22 & 83.77±0.27 & 83.72±0.29 & 84.07±0.23 & 83.56±0.24 & \underline{85.63±0.24} & \textbf{85.86±0.25} \\
    UK & 67.69±0.15 & 64.20±0.17 & 66.16±0.11 & 65.74±0.18 & 66.49±0.12 & 64.83±0.13 & \textbf{69.33±0.13} & \underline{68.55±0.19} \\
    \bottomrule
    \end{tabular}
    \label{tab:res-baseline-br}
\end{table}

\begin{table}[ht]
    \centering
    \caption{Average Wasserstein distance of flow models across datasets. Bold entries mark the best algorithm and underline the second best for each dataset and metric. Note that this table's results uses Euler ODE integration with 100 steps.}
    \begin{tabular}{lcccccccc}
    \toprule
    Dataset & FM-OT & FM-VP & SM-OT & SM-VP & NM-OT & NM-VP & VFM-OT & VFM-VP \\
    \midrule
    Adult & 0.82±0.01 & 0.86±0.00 & 0.85±0.00 & 0.88±0.00 & 0.86±0.00 & 0.89±0.00 & \textbf{0.79±0.00}& \underline{0.80±0.01} \\
    Canada & 2.60±0.01 & 2.73±0.01 & 2.83±0.01 & 2.80±0.01 & 2.81±0.01 & 2.78±0.01 & \textbf{2.47±0.01} & \underline{2.49±0.01} \\
    Churn & \textbf{0.16±0.00}& 0.19±0.01 & 0.17±0.00 & 0.17±0.00 & 0.17±0.01 & 0.18±0.01 & 0.16±0.00 & \underline{0.16±0.00}\\
    Fiji & 2.60±0.00 & 2.73±0.00 & 2.68±0.00 & 2.76±0.00 & 2.69±0.00 & 2.78±0.00 & \textbf{2.52±0.00}& \underline{2.55±0.00}\\
    Indonesia & \underline{0.18±0.00}& 0.19±0.00 & 0.18±0.00 & 0.19±0.00 & 0.18±0.00 & 0.19±0.00 & \textbf{0.17±0.00}& 0.18±0.00 \\
    Rwanda & 0.74±0.00 & 0.79±0.00 & 0.85±0.00 & 0.89±0.01 & 0.87±0.01 & 0.88±0.03 & \textbf{0.69±0.00}& \underline{0.70±0.00}\\
    UK & 1.78±0.00 & 1.86±0.00 & 1.83±0.00 & 1.83±0.00 & 1.82±0.00 & 1.86±0.00 & \textbf{1.75±0.00}& \underline{1.77±0.00}\\
    \bottomrule
    \end{tabular}
    \label{tab:res-baseline-wd}
\end{table}

\begin{table}[ht]
    \centering
    \caption{Average utility of flow models across datasets. Bold entries mark the best algorithm and underline the second best for each dataset and metric. Note that this table's results uses Midpoint ODE integration with 100 steps.}
    \begin{tabular}{lcccccccc}
    \toprule
    Dataset & FM-OT & FM-VP & SM-OT & SM-VP & NM-OT & NM-VP & VFM-OT & VFM-VP \\
    \midrule
    Adult & \textbf{0.75±0.01} & 0.70±0.02 & 0.71±0.02 & 0.70±0.01 & 0.71±0.01 & 0.71±0.01 & 0.74±0.02 & \underline{0.75±0.02} \\
    Canada & 0.70±0.02 & 0.64±0.01 & 0.64±0.01 & 0.63±0.01 & 0.64±0.01 & 0.64±0.01 & \underline{0.71±0.02} & \textbf{0.73±0.02} \\
    Churn & \textbf{0.89±0.02} & 0.80±0.02 & 0.83±0.02 & 0.84±0.02 & 0.84±0.02 & 0.86±0.02 & 0.86±0.02 & \underline{0.87±0.02} \\
    Fiji & 0.70±0.03 & 0.63±0.02 & 0.67±0.02 & 0.64±0.02 & 0.66±0.02 & 0.63±0.02 & \textbf{0.71±0.02} & \underline{0.71±0.02} \\
    Indonesia & 0.89±0.01 & 0.81±0.03 & 0.86±0.02 & 0.83±0.04 & 0.87±0.02 & 0.83±0.03 & \textbf{0.91±0.01} & \underline{0.90±0.03} \\
    Rwanda & \underline{0.68±0.03} & 0.59±0.02 & 0.62±0.02 & 0.61±0.02 & 0.63±0.02 & 0.63±0.02 & 0.67±0.03 & \textbf{0.69±0.03} \\
    UK & 0.76±0.02 & 0.69±0.02 & 0.71±0.02 & 0.72±0.01 & 0.73±0.02 & 0.69±0.01 & \textbf{0.80±0.02} & \underline{0.79±0.02} \\
\bottomrule
\end{tabular}
    \label{tab:res-baseline-u-mp}
\end{table}

\begin{table}[ht]
    \centering
    \caption{Average risk of flow models across datasets. Bold entries mark the best algorithm and underline the second best for each dataset and metric. Note that this table's results uses Midpoint ODE integration with 100 steps.}
    \begin{tabular}{lcccccccc}
    \toprule
    Dataset & FM-OT & FM-VP & SM-OT & SM-VP & NM-OT & NM-VP & VFM-OT & VFM-VP \\
    \midrule
    Adult & 0.47±0.01 & \textbf{0.42±0.02} & 0.45±0.01 & \underline{0.42±0.02} & 0.44±0.02 & 0.43±0.02 & 0.49±0.01 & 0.49±0.02 \\
    Canada & 0.24±0.01 & \textbf{0.21±0.02} & 0.22±0.02 & 0.23±0.02 & \underline{0.21±0.01} & 0.23±0.02 & 0.29±0.01 & 0.30±0.01 \\
    Churn & \textbf{0.03±0.05} & \underline{0.04±0.06} & 0.06±0.08 & 0.05±0.06 & 0.05±0.07 & 0.06±0.07 & 0.05±0.06 & 0.06±0.06 \\
    Fiji & 0.55±0.01 & \textbf{0.53±0.01} & 0.54±0.01 & \underline{0.53±0.01} & 0.54±0.01 & 0.54±0.01 & 0.57±0.01 & 0.57±0.01 \\
    Indonesia & 0.66±0.01 & \textbf{0.64±0.01} & \underline{0.65±0.02} & 0.66±0.01 & 0.66±0.01 & 0.65±0.02 & 0.67±0.01 & 0.66±0.01 \\
    Rwanda & 0.49±0.01 & \underline{0.47±0.01} & 0.47±0.02 & \textbf{0.46±0.02} & 0.47±0.01 & 0.48±0.01 & 0.52±0.01 & 0.52±0.01 \\
    UK & 0.49±0.01 & \textbf{0.47±0.01} & \underline{0.47±0.01} & 0.48±0.01 & 0.47±0.01 & 0.47±0.01 & 0.50±0.01 & 0.49±0.01 \\
    \bottomrule
    \end{tabular}
    \label{tab:res-baseline-r-mp}
\end{table}

\begin{table}[ht]
    \centering
    \caption{Average error in \textbf{shape} (\%) of flow models across datasets. Bold entries mark the best algorithm and underline the second best for each dataset and metric. Note that this table's results uses Midpoint ODE integration with 100 steps.}
    \begin{tabular}{lcccccccc}
    \toprule
    Dataset & FM-OT & FM-VP & SM-OT & SM-VP & NM-OT & NM-VP & VFM-OT & VFM-VP \\
    \midrule
    Adult & \textbf{0.83±0.05} & 1.23±0.05 & \underline{0.98±0.09} & 1.23±0.09 & 1.18±0.07 & 1.25±0.12 & 1.20±0.04 & 1.07±0.04 \\
    Canada & \textbf{3.71±0.04} & 4.54±0.06 & 4.55±0.06 & 4.49±0.05 & 4.53±0.04 & 4.48±0.10 & 3.89±0.05 & \underline{3.73±0.05} \\
    Churn & \textbf{0.91±0.07} & 2.40±0.11 & 1.93±0.26 & 1.57±0.13 & 1.88±0.22 & 1.35±0.15 & 1.35±0.10 & \underline{1.20±0.10} \\
    Fiji & \textbf{0.99±0.04} & 1.85±0.06 & 1.36±0.05 & 1.56±0.04 & 1.18±0.05 & 1.68±0.05 & 1.22±0.05 & \underline{1.13±0.05} \\
    Indonesia & \underline{0.18±0.03} & 0.30±0.02 & 0.24±0.03 & 0.26±0.04 & 0.35±0.05 & 0.26±0.04 & \textbf{0.17±0.02} & 0.18±0.02 \\
    Rwanda & \textbf{0.88±0.05} & 1.69±0.05 & 1.06±0.04 & 1.39±0.06 & 1.17±0.04 & 1.19±0.05 & 1.02±0.04 & \underline{0.94±0.05} \\
    UK & \textbf{0.74±0.04} & 1.48±0.04 & 0.86±0.04 & 0.88±0.03 & \underline{0.77±0.04} & 0.99±0.03 & 0.81±0.05 & 0.80±0.04 \\
    \bottomrule
    \end{tabular}
    \label{tab:res-baseline-sh-mp}
\end{table}

\begin{table}[ht]
    \centering
    \caption{Average error in \textbf{trend} (\%) of flow models across datasets. Bold entries mark the best algorithm and underline the second best for each dataset and metric. Note that this table's results uses Midpoint ODE integration with 100 steps.}
    \begin{tabular}{lcccccccc}
    \toprule
    Dataset & FM-OT & FM-VP & SM-OT & SM-VP & NM-OT & NM-VP & VFM-OT & VFM-VP \\
    \midrule
    Adult & \textbf{2.01±0.05} & 2.82±0.10 & \underline{2.19±0.17} & 2.71±0.17 & 2.36±0.08 & 2.56±0.13 & 2.58±0.07 & 2.28±0.11 \\
    Canada & 8.13±0.72 & 9.43±0.52 & 8.87±0.61 & 8.00±0.20 & 9.05±0.76 & 8.16±0.16 & \textbf{6.72±0.14} & \underline{6.82±0.34} \\
    Churn & \textbf{2.24±0.13} & 4.56±0.17 & 3.65±0.43 & 3.07±0.13 & 3.45±0.29 & \underline{2.73±0.14} & 3.08±0.17 & 2.86±0.18 \\
    Fiji & 1.97±0.04 & 3.25±0.04 & 2.59±0.03 & 2.89±0.04 & 2.47±0.04 & 2.99±0.04 & \underline{1.93±0.04} & \textbf{1.86±0.03} \\
    Indonesia & 0.36±0.03 & 0.65±0.03 & 0.51±0.04 & 0.55±0.03 & 0.62±0.06 & 0.52±0.04 & \textbf{0.32±0.03} & \underline{0.35±0.03} \\
    Rwanda & \textbf{1.27±0.08} & 2.36±0.06 & 2.22±0.08 & 2.22±0.19 & 2.21±0.06 & 1.98±0.08 & 1.46±0.06 & \underline{1.36±0.08} \\
    UK & 1.69±0.05 & 3.05±0.07 & 2.06±0.09 & 2.08±0.03 & 1.98±0.05 & 2.38±0.05 & \underline{1.66±0.05} & \textbf{1.65±0.04} \\
    \bottomrule
    \end{tabular}
    \label{tab:res-baseline-t-mp}
\end{table}

\begin{table}[ht]
    \centering
    \caption{Average \textbf{$\alpha$-precision} (\%) of flow models across datasets. Bold entries mark the best algorithm and underline the second best for each dataset and metric. Note that this table's results uses Midpoint ODE integration with 100 steps.}
    \begin{tabular}{lcccccccc}
    \toprule
    Dataset & FM-OT & FM-VP & SM-OT & SM-VP & NM-OT & NM-VP & VFM-OT & VFM-VP \\
    \midrule
    Adult & 98.77±0.72 & 99.22±0.09 & 98.75±0.21 & \underline{99.35±0.09} & 97.93±0.51 & \textbf{99.48±0.11} & 99.22±0.15 & 99.11±0.17 \\
    Canada & 98.69±0.19 & 98.45±0.21 & 94.62±0.30 & \underline{98.75±0.21} & 95.49±0.21 & \textbf{98.86±0.13} & 97.99±0.22 & 97.97±0.22 \\
    Churn & 98.33±0.84 & 97.85±0.66 & 97.04±0.72 & \textbf{99.07±0.23} & 98.17±0.88 & 98.28±0.44 & 98.92±0.50 & \underline{99.07±0.40} \\
    Fiji & 99.25±0.15 & 98.77±0.17 & \textbf{99.46±0.64} & 99.02±0.19 & \underline{99.44±0.12} & 99.10±0.17 & 99.24±0.16 & 99.25±0.16 \\
    Indonesia & 99.63±0.10 & 99.51±0.10 & 99.67±0.06 & 99.34±0.1 & 99.61±0.08 & 99.47±0.11 & \textbf{99.81±0.06} & \underline{99.79±0.07} \\
    Rwanda & \underline{98.78±0.15} & 98.13±0.23 & 97.93±0.52 & 98.41±0.24 & 98.34±0.29 & \textbf{98.99±0.13} & 98.59±0.26 & 98.70±0.26 \\
    UK & \underline{99.40±0.16} & 99.09±0.15 & 99.37±0.15 & 98.86±0.16 & 99.20±0.16 & 99.22±0.15 & \textbf{99.5±0.15} & 99.32±0.16 \\
    \bottomrule
    \end{tabular}
    \label{tab:res-baseline-ap-mp}
\end{table}

\begin{table}[ht]
    \centering
    \caption{Average \textbf{$\beta$-recall} (\%) of flow models across datasets. Bold entries mark the best algorithm and underline the second best for each dataset and metric. Note that this table's results uses Midpoint ODE integration with 100 steps.}
    \begin{tabular}{lcccccccc}
    \toprule
    Dataset & FM-OT & FM-VP & SM-OT & SM-VP & NM-OT & NM-VP & VFM-OT & VFM-VP \\
    \midrule
    Adult & 46.74±0.71 & 45.94±0.24 & 46.09±0.22 & 45.77±0.26 & 45.91±0.57 & 46.14±0.24 & \underline{49.09±0.23} & \textbf{49.35±0.25} \\
    Canada & 32.18±0.24 & 29.6±0.24 & 28.82±0.26 & 29.78±0.26 & 28.85±0.26 & 29.95±0.36 & \underline{36.45±0.24} & \textbf{36.51±0.31} \\
    Churn & 49.76±0.93 & 50.37±0.7 & 48.68±0.49 & 50.1±0.56 & 49.63±1.08 & 49.78±0.71 & \underline{51.73±0.57} & \textbf{51.79±0.47} \\
    Fiji & 55.62±0.19 & 52.1±0.17 & 53.11±0.69 & 52.25±0.16 & 53.27±0.22 & 52.38±0.21 & \textbf{58.84±0.24} & \underline{58.74±0.25} \\
    Indonesia & 96.44±0.06 & 96.3±0.06 & 96.49±0.04 & 96.39±0.07 & \underline{96.52±0.05} & 96.35±0.06 & \textbf{96.54±0.05} & 96.51±0.04 \\
    Rwanda & \underline{85.47±0.23} & 83.9±0.24 & 83.87±0.26 & 83.96±0.26 & 84.16±0.21 & 83.82±0.13 & 85.36±0.24 & \textbf{85.82±0.24} \\
    UK & 67.69±0.17 & 64.96±0.13 & 66.21±0.13 & 66.35±0.14 & 66.47±0.15 & 65.57±0.15 & \textbf{69.53±0.11} & \underline{69.32±0.17} \\
    \bottomrule
    \end{tabular}
    \label{tab:res-baseline-br-mp}
\end{table}

\begin{table}[ht]
    \centering
    \caption{Average Wasserstein distance of flow models across datasets. Bold entries mark the best algorithm and underline the second best for each dataset and metric. Note that this table's results uses Midpoint ODE integration with 100 steps.}
    \begin{tabular}{lcccccccc}
    \toprule
    Dataset & FM-OT & FM-VP & SM-OT & SM-VP & NM-OT & NM-VP & VFM-OT & VFM-VP \\
    \midrule
    Adult & 0.83±0.02 & 0.85±0.00 & 0.85±0.00 & 0.85±0.00 & 0.85±0.01 & 0.84±0.00 & \textbf{0.79±0.00}& \underline{0.79±0.00}\\
    Canada & 2.59±0.01 & 2.69±0.01 & 2.79±0.01 & 2.69±0.00 & 2.79±0.00 & 2.7±0.00 & \textbf{2.44±0.00}& \underline{2.44±0.00}\\
    Churn & 0.17±0.01 & 0.18±0.01 & 0.17±0.00 & 0.17±0.00 & 0.17±0.01 & 0.17±0.01 & \textbf{0.16±0.00}& \underline{0.16±0.00}\\
    Fiji & 2.59±0.00 & 2.67±0.00 & 2.68±0.02 & 2.67±0.00 & 2.68±0.00 & 2.67±0.00 & \underline{2.51±0.00}& \textbf{2.50±0.00}\\
    Indonesia & 0.18±0.00 & 0.18±0.00 & 0.18±0.00 & 0.18±0.00 & 0.18±0.00 & 0.18±0.00 & \textbf{0.17±0.00}& \underline{0.17±0.00}\\
    Rwanda & 0.73±0.00 & 0.77±0.01 & 0.84±0.01 & 0.79±0.00 & 0.84±0.00 & 0.8±0.00 & \textbf{0.69±0.00}& \underline{0.69±0.00}\\
    UK & 1.77±0.00 & 1.83±0.00 & 1.82±0.00 & 1.80±0.00 & 1.82±0.00 & 1.82±0.00 & \textbf{1.74±0.00}& \underline{1.74±0.00}\\
    \bottomrule
    \end{tabular}
    \label{tab:res-baseline-wd-mp}
\end{table}

\clearpage
\subsection{Detailed Evaluation Results based on Data - SDE}

We evaluate the same set of models under SDE sampling with 100 integration steps. Tables~\ref{tab:res-baseline-u-sde}--\ref{tab:res-baseline-wd-sde} report results using Euler-Maruyama, and Tables~\ref{tab:res-baseline-u-sde-mp}--\ref{tab:res-baseline-wd-sde-mp} report results using the Midpoint solver. All entries are mean $\pm$ standard deviation across 20 seeds. Some values may coincide at the reported precision due to rounding to two decimal places.

\begin{table}[ht]
    \centering
    \caption{Average utility of flow models across datasets. Bold entries mark the best algorithm and underline the second best for each dataset and metric. Note that this table's results uses Euler-Maruyama SDE integration with 100 steps.}
    \begin{tabular}{lcccccccc}
    \toprule
    Dataset & FM-OT & FM-VP & SM-OT & SM-VP & NM-OT & NM-VP & VFM-OT & VFM-VP \\
    \midrule
    Adult & \underline{0.75±0.02} & 0.69±0.01 & 0.72±0.01 & 0.69±0.01 & 0.73±0.01 & 0.71±0.02 & 0.75±0.01 & \textbf{0.76±0.01} \\
    Canada & 0.70±0.01 & 0.65±0.01 & 0.65±0.01 & 0.63±0.01 & 0.66±0.01 & 0.64±0.01 & \underline{0.71±0.02} & \textbf{0.73±0.02} \\
    Churn & \textbf{0.88±0.02} & 0.80±0.02 & 0.86±0.02 & 0.85±0.01 & 0.86±0.02 & 0.86±0.01 & 0.85±0.02 & \underline{0.87±0.01} \\
    Fiji & \underline{0.70±0.02} & 0.62±0.01 & 0.66±0.02 & 0.63±0.02 & 0.65±0.02 & 0.62±0.02 & \textbf{0.71±0.02} & 0.70±0.02 \\
    Indonesia & \textbf{0.89±0.02} & 0.79±0.01 & 0.86±0.02 & 0.81±0.02 & \underline{0.89±0.01} & 0.81±0.02 & 0.89±0.03 & 0.87±0.03 \\
    Rwanda & \textbf{0.68±0.02} & 0.58±0.01 & 0.62±0.02 & 0.63±0.02 & 0.63±0.01 & 0.63±0.02 & 0.66±0.02 & \underline{0.67±0.02} \\
    UK & 0.75±0.02 & 0.68±0.01 & 0.72±0.02 & 0.7±0.02 & 0.73±0.02 & 0.68±0.01 & \textbf{0.78±0.02} & \underline{0.77±0.02} \\
\bottomrule
\end{tabular}
    \label{tab:res-baseline-u-sde}
\end{table}

\begin{table}[ht]
    \centering
    \caption{Average risk of flow models across datasets. Bold entries mark the best algorithm and underline the second best for each dataset and metric. Note that this table's results uses Euler-Maruyama SDE integration with 100 steps.}
    \begin{tabular}{lcccccccc}
    \toprule
    Dataset & FM-OT & FM-VP & SM-OT & SM-VP & NM-OT & NM-VP & VFM-OT & VFM-VP \\
    \midrule
    Adult & 0.47±0.02 & \underline{0.41±0.02} & 0.45±0.02 & \textbf{0.39±0.02} & 0.45±0.02 & 0.41±0.01 & 0.49±0.02 & 0.48±0.02 \\
    Canada & 0.24±0.02 & \textbf{0.21±0.02} & 0.22±0.02 & \underline{0.21±0.02} & 0.22±0.02 & 0.22±0.02 & 0.30±0.01 & 0.29±0.01 \\
    Churn & 0.06±0.06 & 0.06±0.07 & \textbf{0.03±0.04} & \underline{0.05±0.06} & 0.08±0.07 & 0.07±0.08 & 0.07±0.07 & 0.07±0.08 \\
    Fiji & 0.55±0.01 & \textbf{0.53±0.01} & 0.54±0.01 & \underline{0.53±0.01} & 0.54±0.01 & 0.53±0.01 & 0.56±0.01 & 0.56±0.01 \\
    Indonesia & 0.66±0.01 & \textbf{0.64±0.01} & \underline{0.65±0.02} & 0.65±0.02 & 0.65±0.01 & 0.65±0.01 & 0.66±0.01 & 0.66±0.01 \\
    Rwanda & 0.50±0.01 & \underline{0.47±0.01} & 0.47±0.02 & \textbf{0.45±0.02} & 0.48±0.01 & 0.47±0.01 & 0.52±0.01 & 0.52±0.01 \\
    UK & 0.48±0.01 & \textbf{0.46±0.01} & 0.48±0.01 & \underline{0.47±0.01} & 0.48±0.01 & 0.47±0.00 & 0.49±0.01 & 0.49±0.01 \\
    \bottomrule
    \end{tabular}
    \label{tab:res-baseline-r-sde}
\end{table}

\begin{table}[ht]
    \centering
    \caption{Average error in \textbf{shape} (\%) of flow models across datasets. Bold entries mark the best algorithm and underline the second best for each dataset and metric. Note that this table's results uses Euler-Maruyama SDE integration with 100 steps.}
    \begin{tabular}{lcccccccc}
    \toprule
    Dataset & FM-OT & FM-VP & SM-OT & SM-VP & NM-OT & NM-VP & VFM-OT & VFM-VP \\
    \midrule
    Adult & \textbf{0.83±0.04} & 1.30±0.06 & \underline{0.83±0.05} & 1.31±0.06 & 0.96±0.09 & 1.46±0.06 & 1.17±0.04 & 1.04±0.04 \\
    Canada & \underline{3.71±0.05} & 4.55±0.06 & 4.33±0.12 & 4.69±0.06 & 4.20±0.05 & 4.60±0.05 & 3.83±0.05 & \textbf{3.69±0.08} \\
    Churn & \textbf{0.90±0.10} & 2.53±0.09 & 1.27±0.16 & 1.56±0.16 & 1.28±0.19 & 1.49±0.24 & 1.38±0.10 & \underline{1.26±0.11} \\
    Fiji & \textbf{1.01±0.04} & 1.91±0.05 & 1.33±0.04 & 1.89±0.06 & 1.27±0.05 & 2.03±0.06 & 1.23±0.05 & \underline{1.17±0.06} \\
    Indonesia & \textbf{0.21±0.03} & 0.33±0.04 & 0.27±0.04 & 0.32±0.05 & 0.24±0.03 & 0.35±0.05 & \underline{0.21±0.03} & 0.23±0.04 \\
    Rwanda & \underline{0.83±0.05} & 1.66±0.04 & 1.01±0.04 & 1.28±0.04 & 1.06±0.05 & 1.35±0.05 & 0.98±0.04 & \textbf{0.78±0.05} \\
    UK & \underline{0.77±0.04} & 1.43±0.04 & 0.84±0.03 & 0.94±0.07 & 0.77±0.03 & 1.02±0.03 & 0.80±0.04 & \textbf{0.73±0.03} \\
    \bottomrule
    \end{tabular}
    \label{tab:res-baseline-sh-sde}
\end{table}

\begin{table}[ht]
    \centering
    \caption{Average error in \textbf{trend} (\%) of flow models across datasets. Bold entries mark the best algorithm and underline the second best for each dataset and metric. Note that this table's results uses Euler-Maruyama SDE integration with 100 steps.}
    \begin{tabular}{lcccccccc}
    \toprule
    Dataset & FM-OT & FM-VP & SM-OT & SM-VP & NM-OT & NM-VP & VFM-OT & VFM-VP \\
    \midrule
    Adult & \textbf{2.03±0.06} & 2.90±0.07 & \underline{2.04±0.06} & 2.85±0.09 & 2.19±0.14 & 2.81±0.08 & 2.50±0.05 & 2.29±0.16 \\
    Canada & 8.11±0.74 & 9.66±0.49 & 8.98±0.73 & 8.44±0.21 & 9.06±0.78 & 8.23±0.13 & \underline{7.49±2.28} & \textbf{6.81±0.30} \\
    Churn & \textbf{2.32±0.22} & 4.71±0.13 & 2.55±0.28 & 3.00±0.22 & \underline{2.54±0.31} & 2.82±0.28 & 3.12±0.19 & 2.88±0.21 \\
    Fiji & \underline{2.05±0.04} & 3.46±0.06 & 2.59±0.05 & 3.40±0.06 & 2.54±0.04 & 3.58±0.05 & \textbf{2.00±0.04} & 2.05±0.04 \\
    Indonesia & \textbf{0.41±0.04} & 0.73±0.04 & 0.54±0.04 & 0.66±0.05 & 0.48±0.04 & 0.65±0.05 & \underline{0.41±0.04} & 0.46±0.05 \\
    Rwanda & \underline{1.25±0.08} & 2.38±0.05 & 2.07±0.08 & 2.62±0.12 & 2.05±0.10 & 2.44±0.11 & 1.43±0.06 & \textbf{1.24±0.07} \\
    UK & 1.75±0.06 & 3.18±0.05 & 2.06±0.05 & 2.35±0.13 & 1.97±0.04 & 2.59±0.05 & \textbf{1.70±0.04} & \underline{1.74±0.04} \\
    \bottomrule
    \end{tabular}
    \label{tab:res-baseline-t-sde}
\end{table}

\begin{table}[ht]
    \centering
    \caption{Average \textbf{$\alpha$-precision} (\%) of flow models across datasets. Bold entries mark the best algorithm and underline the second best for each dataset and metric. Note that this table's results uses Euler-Maruyama SDE integration with 100 steps.}
    \begin{tabular}{lcccccccc}
    \toprule
    Dataset & FM-OT & FM-VP & SM-OT & SM-VP & NM-OT & NM-VP & VFM-OT & VFM-VP \\
    \midrule
    Adult & \underline{99.33±0.13} & 99.16±0.12 & 98.94±0.20 & 98.86±0.18 & 98.82±0.19 & 97.6±0.23 & 99.21±0.15 & \textbf{99.38±0.09} \\
    Canada & \textbf{98.62±0.15} & 98.38±0.12 & 95.74±0.21 & 95.83±0.20 & 96.92±0.19 & 96.2±0.20 & 98.25±0.18 & \underline{98.61±0.17} \\
    Churn & \textbf{99.29±0.15} & 97.77±0.48 & 98.02±0.46 & 98.32±0.45 & 98.19±0.60 & 97.90±0.42 & 98.89±0.35 & \underline{99.12±0.31} \\
    Fiji & 99.36±0.16 & 99.28±0.13 & 99.24±0.19 & 98.59±0.13 & \textbf{99.54±0.10} & 97.96±0.15 & 99.08±0.20 & \underline{99.49±0.10} \\
    Indonesia & 99.49±0.10 & 99.58±0.06 & 99.43±0.10 & \underline{99.68±0.06} & 99.50±0.10 & 99.62±0.07 & 99.50±0.11 & \textbf{99.82±0.03} \\
    Rwanda & \underline{98.85±0.14} & 98.09±0.20 & 98.41±0.31 & 97.26±0.44 & 98.54±0.26 & 97.25±0.54 & 98.54±0.28 & \textbf{98.98±0.18} \\
    UK & 99.29±0.13 & 99.54±0.04 & \textbf{99.67±0.06} & 99.58±0.04 & \underline{99.63±0.15} & 99.53±0.12 & 99.42±0.13 & 99.60±0.13 \\
    \bottomrule
    \end{tabular}
    \label{tab:res-baseline-ap-sde}
\end{table}

\begin{table}[ht]
    \centering
    \caption{Average \textbf{$\beta$-recall} (\%) of flow models across datasets. Bold entries mark the best algorithm and underline the second best for each dataset and metric. Note that this table's results uses Euler-Maruyama SDE integration with 100 steps.}
    \begin{tabular}{lcccccccc}
    \toprule
    Dataset & FM-OT & FM-VP & SM-OT & SM-VP & NM-OT & NM-VP & VFM-OT & VFM-VP \\
    \midrule
    Adult & 47.45±0.19 & 45.39±0.19 & 46.50±0.26 & 45.19±0.28 & 46.35±0.30 & 44.93±0.21 & \textbf{49.16±0.23} & \underline{48.78±0.21} \\
    Canada & 32.19±0.24 & 29.18±0.29 & 29.41±0.27 & 28.69±0.32 & 29.36±0.30 & 28.71±0.27 & \textbf{36.09±0.24} & \underline{35.67±0.29} \\
    Churn & 50.56±0.39 & 50.89±0.77 & 49.31±0.46 & 49.70±0.58 & 49.66±0.38 & 49.63±0.71 & \underline{51.52±0.49} & \textbf{51.67±0.47} \\
    Fiji & 55.62±0.23 & 50.79±0.19 & 53.54±0.25 & 50.68±0.25 & 53.52±0.24 & 50.50±0.21 & \textbf{58.55±0.17} & \underline{57.46±0.19} \\
    Indonesia & 96.45±0.06 & 96.33±0.07 & 96.46±0.05 & 96.41±0.05 & 96.48±0.05 & 96.41±0.05 & \underline{96.55±0.05} & \textbf{96.58±0.03} \\
    Rwanda & 85.49±0.22 & 83.74±0.21 & 84.14±0.15 & 83.77±0.24 & 84.31±0.25 & 83.72±0.18 & \underline{85.72±0.22} & \textbf{85.85±0.16} \\
    UK & 67.70±0.18 & 64.19±0.13 & 66.34±0.18 & 65.78±0.16 & 66.73±0.26 & 64.95±0.16 & \textbf{69.31±0.16} & \underline{68.56±0.13} \\
    \bottomrule
    \end{tabular}
    \label{tab:res-baseline-br-sde}
\end{table}

\begin{table}[ht]
    \centering
    \caption{Average Wasserstein distance (\%) of flow models across datasets. Bold entries mark the best algorithm and underline the second best for each dataset and metric. Note that this table's results uses Euler-Maruyama integration with 100 steps.}
    \begin{tabular}{lcccccccc}
    \toprule
    Dataset & FM-OT & FM-VP & SM-OT & SM-VP & NM-OT & NM-VP & VFM-OT & VFM-VP \\
    \midrule
    Adult & 0.82±0.00 & 0.86±0.00 & 0.84±0.00 & 0.87±0.00 & 0.85±0.01 & 0.88±0.00 & \textbf{0.79±0.00}& \underline{0.80±0.00}\\
    Canada & 2.59±0.00 & 2.73±0.00 & 2.77±0.01 & 2.78±0.01 & 2.76±0.01 & 2.76±0.01 & \textbf{2.46±0.01} & \underline{2.48±0.00}\\
    Churn & \textbf{0.16±0.00}& 0.19±0.01 & 0.17±0.00 & 0.17±0.00 & 0.17±0.00 & 0.17±0.00 & 0.16±0.00 & \underline{0.16±0.00}\\
    Fiji & 2.60±0.00 & 2.72±0.00 & 2.67±0.00 & 2.74±0.00 & 2.67±0.00 & 2.75±0.00 & \textbf{2.52±0.00}& \underline{2.55±0.00}\\
    Indonesia & \underline{0.18±0.00}& 0.19±0.00 & 0.18±0.00 & 0.19±0.00 & 0.18±0.00 & 0.19±0.00 & \textbf{0.17±0.00}& 0.18±0.00 \\
    Rwanda & 0.74±0.00 & 0.79±0.00 & 0.83±0.00 & 0.85±0.02 & 0.83±0.01 & 0.85±0.01 & \textbf{0.69±0.00}& \underline{0.70±0.00}\\
    UK & 1.78±0.00 & 1.86±0.00 & 1.82±0.00 & 1.83±0.00 & 1.81±0.01 & 1.85±0.00 & \textbf{1.75±0.00}& \underline{1.77±0.00}\\
    \bottomrule
    \end{tabular}
    \label{tab:res-baseline-wd-sde}
\end{table}

\begin{table}[ht]
    \centering
    \caption{Average utility of flow models across datasets. Bold entries mark the best algorithm and underline the second best for each dataset and metric. Note that this table's results uses Midpoint SDE integration with 100 steps.}
    \begin{tabular}{lcccccccc}
    \toprule
    Dataset & FM-OT & FM-VP & SM-OT & SM-VP & NM-OT & NM-VP & VFM-OT & VFM-VP \\
    \midrule
    Adult & \textbf{0.75±0.02} & 0.70±0.01 & 0.73±0.01 & 0.70±0.02 & 0.73±0.01 & 0.71±0.01 & 0.74±0.02 & \underline{0.75±0.01} \\
    Canada & 0.70±0.02 & 0.65±0.01 & 0.65±0.01 & 0.64±0.01 & 0.66±0.01 & 0.64±0.01 & \underline{0.71±0.02} & \textbf{0.74±0.02} \\
    Churn & \textbf{0.89±0.02} & 0.80±0.01 & 0.86±0.02 & 0.85±0.02 & 0.85±0.01 & \underline{0.87±0.02} & 0.85±0.02 & 0.87±0.01 \\
    Fiji & \underline{0.71±0.02} & 0.64±0.01 & 0.68±0.01 & 0.64±0.02 & 0.66±0.02 & 0.63±0.02 & \textbf{0.72±0.02} & 0.71±0.02 \\
    Indonesia & 0.90±0.03 & 0.81±0.03 & 0.88±0.02 & 0.84±0.03 & 0.89±0.02 & 0.83±0.02 & \textbf{0.91±0.02} & \underline{0.91±0.01} \\
    Rwanda & \textbf{0.69±0.02} & 0.59±0.02 & 0.62±0.02 & 0.62±0.02 & 0.63±0.02 & 0.62±0.02 & 0.67±0.02 & \underline{0.68±0.02} \\
    UK & 0.77±0.02 & 0.68±0.01 & 0.73±0.02 & 0.72±0.01 & 0.74±0.02 & 0.69±0.01 & \textbf{0.80±0.02} & \underline{0.79±0.02} \\
\bottomrule
\end{tabular}
    \label{tab:res-baseline-u-sde-mp}
\end{table}

\begin{table}[ht]
    \centering
    \caption{Average risk of flow models across datasets. Bold entries mark the best algorithm and underline the second best for each dataset and metric. Note that this table's results uses Midpoint SDE integration with 100 steps.}
    \begin{tabular}{lcccccccc}
    \toprule
    Dataset & FM-OT & FM-VP & SM-OT & SM-VP & NM-OT & NM-VP & VFM-OT & VFM-VP \\
    \midrule
    Adult & 0.47±0.02 & \textbf{0.42±0.01} & 0.46±0.02 & \underline{0.42±0.01} & 0.45±0.02 & 0.42±0.01 & 0.50±0.02 & 0.49±0.02 \\
    Canada & 0.24±0.02 & \textbf{0.22±0.02} & \underline{0.22±0.02} & 0.22±0.01 & 0.22±0.01 & 0.22±0.02 & 0.30±0.02 & 0.30±0.01 \\
    Churn & \textbf{0.05±0.06} & \underline{0.05±0.07} & 0.05±0.05 & 0.05±0.07 & 0.08±0.07 & 0.07±0.08 & 0.07±0.07 & 0.07±0.07 \\
    Fiji & 0.55±0.01 & \textbf{0.54±0.01} & \underline{0.54±0.01} & 0.54±0.01 & 0.54±0.01 & 0.54±0.01 & 0.57±0.01 & 0.56±0.01 \\
    Indonesia & 0.66±0.01 & \textbf{0.64±0.01} & \underline{0.65±0.02} & 0.65±0.01 & 0.66±0.01 & 0.65±0.02 & 0.66±0.01 & 0.65±0.01 \\
    Rwanda & 0.50±0.01 & 0.48±0.01 & \underline{0.46±0.02} & \textbf{0.46±0.02} & 0.48±0.01 & 0.48±0.01 & 0.52±0.01 & 0.52±0.01 \\
    UK & 0.49±0.01 & \textbf{0.47±0.01} & 0.48±0.01 & \underline{0.47±0.01} & 0.48±0.01 & 0.47±0.01 & 0.49±0.01 & 0.49±0.01 \\
    \bottomrule
    \end{tabular}
    \label{tab:res-baseline-r-sde-mp}
\end{table}

\begin{table}[ht]
    \centering
    \caption{Average error in \textbf{shape} (\%) of flow models across datasets. Bold entries mark the best algorithm and underline the second best for each dataset and metric. Note that this table's results uses Midpoint SDE integration with 100 steps.}
    \begin{tabular}{lcccccccc}
    \toprule
    Dataset & FM-OT & FM-VP & SM-OT & SM-VP & NM-OT & NM-VP & VFM-OT & VFM-VP \\
    \midrule
    Adult & \underline{0.80±0.05} & 1.23±0.04 & \textbf{0.78±0.09} & 1.23±0.05 & 0.92±0.10 & 1.30±0.20 & 1.21±0.04 & 1.09±0.04 \\
    Canada & \textbf{3.71±0.05} & 4.57±0.06 & 4.30±0.17 & 4.51±0.06 & 4.15±0.04 & 4.47±0.06 & 3.92±0.05 & \underline{3.74±0.05} \\
    Churn & \textbf{0.85±0.10} & 2.43±0.10 & 1.32±0.22 & 1.28±0.24 & 1.28±0.11 & \underline{1.17±0.15} & 1.33±0.10 & 1.20±0.12 \\
    Fiji & \textbf{0.98±0.04} & 1.86±0.05 & 1.23±0.05 & 1.59±0.04 & \underline{1.15±0.05} & 1.68±0.05 & 1.23±0.04 & 1.15±0.04 \\
    Indonesia & \textbf{0.16±0.02} & 0.29±0.02 & 0.21±0.08 & 0.25±0.04 & \underline{0.17±0.03} & 0.24±0.03 & 0.17±0.02 & 0.18±0.03 \\
    Rwanda & \textbf{0.88±0.06} & 1.69±0.05 & 1.02±0.04 & 1.43±0.06 & 1.08±0.05 & 1.27±0.05 & 1.03±0.05 & \underline{0.94±0.05} \\
    UK & \underline{0.70±0.03} & 1.48±0.04 & 0.70±0.03 & 0.87±0.04 & \textbf{0.69±0.08} & 0.97±0.03 & 0.82±0.03 & 0.80±0.02 \\
    \bottomrule
    \end{tabular}
    \label{tab:res-baseline-sh-sde-mp}
\end{table}

\begin{table}[ht]
    \centering
    \caption{Average error in \textbf{trend} (\%) of flow models across datasets. Bold entries mark the best algorithm and underline the second best for each dataset and metric. Note that this table's results uses Midpoint SDE integration with 100 steps.}
    \begin{tabular}{lcccccccc}
    \toprule
    Dataset & FM-OT & FM-VP & SM-OT & SM-VP & NM-OT & NM-VP & VFM-OT & VFM-VP \\
    \midrule
    Adult & \underline{1.98±0.06} & 2.78±0.06 & \textbf{1.94±0.12} & 2.71±0.07 & 2.09±0.10 & 2.63±0.24 & 2.61±0.05 & 2.29±0.12 \\
    Canada & 8.07±0.74 & 9.57±0.43 & 8.90±0.77 & 8.08±0.24 & 8.83±0.78 & 8.11±0.16 & \underline{7.53±2.30} & \textbf{6.81±0.28} \\
    Churn & \textbf{2.24±0.20} & 4.61±0.11 & 2.57±0.32 & 2.72±0.33 & 2.50±0.15 & \underline{2.46±0.17} & 3.08±0.20 & 2.88±0.21 \\
    Fiji & \underline{1.92±0.04} & 3.25±0.04 & 2.42±0.04 & 2.92±0.04 & 2.34±0.04 & 3.00±0.05 & 1.93±0.04 & \textbf{1.88±0.04} \\
    Indonesia & \underline{0.33±0.02} & 0.64±0.02 & 0.44±0.13 & 0.54±0.05 & 0.38±0.04 & 0.49±0.04 & \textbf{0.31±0.02} & 0.33±0.03 \\
    Rwanda & \textbf{1.23±0.06} & 2.35±0.05 & 2.01±0.09 & 2.17±0.07 & 1.87±0.09 & 2.01±0.07 & 1.44±0.05 & \underline{1.35±0.07} \\
    UK & \textbf{1.61±0.04} & 3.04±0.06 & 1.82±0.04 & 2.08±0.04 & 1.84±0.18 & 2.36±0.04 & 1.65±0.05 & \underline{1.64±0.03} \\
    \bottomrule
    \end{tabular}
    \label{tab:res-baseline-t-sde-mp}
\end{table}

\begin{table}[ht]
    \centering
    \caption{Average \textbf{$\alpha$-precision} (\%) of flow models across datasets. Bold entries mark the best algorithm and underline the second best for each dataset and metric. Note that this table's results uses Midpoint SDE integration with 100 steps.}
    \begin{tabular}{lcccccccc}
    \toprule
    Dataset & FM-OT & FM-VP & SM-OT & SM-VP & NM-OT & NM-VP & VFM-OT & VFM-VP \\
    \midrule
    Adult & \underline{99.25±0.33} & 99.20±0.12 & 99.18±0.16 & 99.24±0.11 & 98.87±0.19 & \textbf{99.35±0.11} & 99.11±0.17 & 99.16±0.15 \\
    Canada & 98.55±0.17 & 98.35±0.18 & 96.71±0.19 & \underline{98.65±0.17} & 97.40±0.19 & \textbf{98.76±0.15} & 97.83±0.21 & 97.96±0.20 \\
    Churn & \textbf{99.12±0.54} & 97.97±0.49 & 98.18±0.46 & 98.89±0.65 & 98.40±0.57 & 98.85±0.45 & 98.95±0.32 & \underline{99.10±0.30} \\
    Fiji & 99.16±0.18 & 98.72±0.19 & \underline{99.27±0.18} & 98.81±0.19 & \textbf{99.50±0.11} & 98.91±0.16 & 99.00±0.20 & 99.18±0.18 \\
    Indonesia & 99.63±0.09 & 99.47±0.08 & 99.73±0.05 & 99.34±0.10 & 99.72±0.05 & 99.40±0.10 & \underline{99.75±0.09} & \textbf{99.78±0.07} \\
    Rwanda & 98.71±0.17 & 98.01±0.21 & \textbf{98.82±0.25} & 98.06±0.32 & \underline{98.82±0.16} & 98.72±0.20 & 98.44±0.31 & 98.59±0.29 \\
    UK & 99.33±0.14 & 99.07±0.14 & \underline{99.64±0.08} & 98.90±0.14 & \textbf{99.71±0.05} & 99.16±0.13 & 99.30±0.14 & 99.29±0.14 \\
    \bottomrule
    \end{tabular}
    \label{tab:res-baseline-ap-sde-mp}
\end{table}

\begin{table}[ht]
    \centering
    \caption{Average \textbf{$\beta$-recall} (\%) of flow models across datasets. Bold entries mark the best algorithm and underline the second best for each dataset and metric. Note that this table's results uses Midpoint SDE integration with 100 steps.}
    \begin{tabular}{lcccccccc}
    \toprule
    Dataset & FM-OT & FM-VP & SM-OT & SM-VP & NM-OT & NM-VP & VFM-OT & VFM-VP \\
    \midrule
    Adult & 47.39±0.45 & 45.90±0.19 & 46.61±0.20 & 46.01±0.21 & 46.32±0.28 & 46.21±0.20 & \underline{49.27±0.24} & \textbf{49.29±0.21} \\
    Canada & 32.24±0.24 & 29.57±0.33 & 29.62±0.29 & 30.07±0.28 & 29.52±0.32 & 30.15±0.25 & \underline{36.46±0.22} & \textbf{36.57±0.29} \\
    Churn & 50.22±0.52 & 50.83±0.54 & 49.22±0.43 & 49.71±0.76 & 49.76±0.55 & 49.84±0.66 & \underline{51.57±0.55} & \textbf{51.76±0.49} \\
    Fiji & 56.01±0.20 & 52.30±0.18 & 53.95±0.23 & 52.28±0.43 & 53.91±0.25 & 52.42±0.22 & \textbf{59.13±0.23} & \underline{58.76±0.16} \\
    Indonesia & 96.44±0.07 & 96.29±0.06 & 96.46±0.07 & 96.37±0.05 & 96.49±0.06 & 96.37±0.04 & \textbf{96.53±0.05} & \underline{96.50±0.03} \\
    Rwanda & \underline{85.56±0.23} & 83.88±0.20 & 84.25±0.27 & 83.89±0.26 & 84.42±0.25 & 83.83±0.17 & 85.43±0.18 & \textbf{85.78±0.16} \\
    UK & 67.81±0.33 & 64.95±0.14 & 66.57±0.18 & 66.33±0.22 & 66.82±0.15 & 65.60±0.19 & \textbf{69.61±0.16} & \underline{69.32±0.15} \\
    \bottomrule
    \end{tabular}
    \label{tab:res-baseline-br-sde-mp}
\end{table}

\begin{table}[ht]
    \centering
    \caption{Average Wasserstein distance (\%) of flow models across datasets. Bold entries mark the best algorithm and underline the second best for each dataset and metric. Note that this table's results uses Midpoint SDE integration with 100 steps.}
    \begin{tabular}{lcccccccc}
    \toprule
    Dataset & FM-OT & FM-VP & SM-OT & SM-VP & NM-OT & NM-VP & VFM-OT & VFM-VP \\
    \midrule
    Adult & 0.81±0.01 & 0.85±0.00 & 0.84±0.00 & 0.85±0.00 & 0.84±0.00 & 0.84±0.00 & \textbf{0.78±0.00}& \underline{0.79±0.01} \\
    Canada & 2.57±0.00 & 2.69±0.00 & 2.73±0.01 & 2.69±0.01 & 2.73±0.01 & 2.68±0.01 & \textbf{2.43±0.01} & \underline{2.44±0.00}\\
    Churn & \textbf{0.16±0.00}& 0.20±0.00 & 0.17±0.00 & 0.17±0.00 & 0.17±0.00 & 0.17±0.01 & 0.16±0.00 & \underline{0.16±0.00}\\
    Fiji & 2.58±0.00 & 2.67±0.00 & 2.65±0.00 & 2.67±0.02 & 2.65±0.00 & 2.67±0.00 & \textbf{2.50±0.00}& \underline{2.50±0.00}\\
    Indonesia & 0.18±0.00 & 0.18±0.00 & 0.18±0.00 & 0.18±0.00 & 0.18±0.00 & 0.18±0.00 & \textbf{0.17±0.00}& \underline{0.17±0.00}\\
    Rwanda & 0.73±0.00 & 0.77±0.00 & 0.81±0.00 & 0.79±0.00 & 0.81±0.01 & 0.79±0.00 & \textbf{0.68±0.00}& \underline{0.69±0.00}\\
    UK & 1.77±0.01 & 1.83±0.00 & 1.81±0.00 & 1.80±0.00 & 1.80±0.00 & 1.82±0.00 & \textbf{1.73±0.00}& \underline{1.74±0.00}\\
    \bottomrule
    \end{tabular}
    \label{tab:res-baseline-wd-sde-mp}
\end{table}

\end{document}